\documentclass[11pt]{article}
\usepackage{fullpage}
\usepackage{authblk}
\usepackage{microtype}
\usepackage{graphicx}
\usepackage{subfigure}
\usepackage{booktabs} 
\usepackage{multirow}
\usepackage[table]{xcolor}
\usepackage{hhline}
\usepackage{natbib}

\usepackage{amsmath}
\usepackage{amssymb}
\usepackage{mathtools}
\usepackage{amsthm}

\newcommand{\declarecolor}[2]{\definecolor{#1}{RGB}{#2}\expandafter\newcommand\csname #1\endcsname[1]{\textcolor{#1}{##1}}}
\declarecolor{White}{255, 255, 255}
\declarecolor{Black}{0, 0, 0}
\declarecolor{Maroon}{128, 0, 0}
\declarecolor{Coral}{255, 127, 80}
\declarecolor{Red}{182, 21, 21}
\declarecolor{LimeGreen}{50, 205, 50}
\declarecolor{DarkGreen}{0, 80, 0}
\declarecolor{Purple}{146, 42, 158}
\declarecolor{Navy}{0, 0, 128}
\declarecolor{LightBlue}{84, 101, 202}
\usepackage{hyperref}
\hypersetup{
    colorlinks,
    citecolor=DarkGreen,
    linkcolor=LightBlue}
\usepackage[noabbrev, capitalise, nameinlink]{cleveref}

\usepackage[textsize=tiny]{todonotes}

\newtheorem{theorem}{Theorem}
\newtheorem*{theorem*}{Theorem}
\newtheorem{lemma}{Lemma}

\newtheorem{assumption}{Assumption}

\newtheorem{claim}{Claim}
\newtheorem{proposition}{Proposition}
\newtheorem{remark}{Remark}

\usepackage{pifont}

\DeclareMathOperator*{\argmin}{argmin}
\DeclareMathOperator*{\argmax}{argmax}

\DeclareMathOperator{\DualGap}{DualityGap}
\DeclareMathOperator{\KL}{KL}

\newcommand{\notshow}[1]{{}}
\newcommand{\AutoAdjust}[3]{{ \mathchoice{ \left #1 #2  \right #3}{#1 #2 #3}{#1 #2 #3}{#1 #2 #3} }}
\newcommand{\Xcomment}[1]{{}}

\newcommand{\InParentheses}[1]{\AutoAdjust{(}{#1}{)}}
\newcommand{\InBrackets}[1]{\AutoAdjust{[}{#1}{]}}
\newcommand{\InAngles}[1]{\AutoAdjust{\langle}{#1}{\rangle}}
\newcommand{\InNorms}[1]{\AutoAdjust{\|}{#1}{\|}}
\renewcommand{\part}[2]{\frac{\partial #1}{\partial #2}}

\newcommand{\X}{\mathcal{X}}
\newcommand{\Y}{\mathcal{Y}}

\newcommand{\hz}{\widehat{z}}

\def\+#1{\mathcal{#1}}
\def\-#1{\mathbb{#1}}

\usepackage{enumitem}
\usetikzlibrary{patterns,shapes.geometric,patterns.meta,decorations.pathmorphing}

\title{On Separation Between Best-Iterate, Random-Iterate, and Last-Iterate Convergence of Learning in Games\thanks{Authors are listed in alphabetical order.}}
\author[1]{Yang Cai}
\author[2]{Gabriele Farina}
\author[3]{Julien Grand-Cl{\'e}ment}
\author[4]{Christian Kroer}
\author[5]{Chung-Wei Lee}
\author[6]{Haipeng Luo}
\author[1]{Weiqiang Zheng}

\affil[1]{Yale University, \texttt{\{yang.cai,weiqiang.zheng\}@yale.edu}}
\affil[2]{MIT, \texttt{gfarina@mit.edu}}
\affil[3]{HEC Paris, \texttt{grand-clement@hec.fr}}
\affil[4]{Columbia University, \texttt{christian.kroer@columbia.edu}}
\affil[5]{\texttt{leechung@usc.edu}}
\affil[6]{University of Southern California, \texttt{haipengl@usc.edu}}

\begin{document}
\maketitle

\begin{abstract}
Non-ergodic convergence of learning dynamics in games is widely studied recently because of its importance in both theory and practice. 
Recent work~\citep{cai2024fast} showed that a broad class of learning dynamics, including Optimistic Multiplicative Weights Update (OMWU), can exhibit arbitrarily slow last-iterate convergence even in simple $2 \times 2$ matrix games, despite many of these dynamics being known to converge asymptotically in the last iterate. It remains unclear, however, whether these algorithms achieve fast non-ergodic convergence under weaker criteria, such as best-iterate convergence.
We show that for $2\times 2$ matrix games, OMWU achieves an $O(T^{-1/6})$ best-iterate convergence rate, in stark contrast to its slow last-iterate convergence in the same class of games. 
Furthermore, we establish a lower bound showing that OMWU does not achieve any polynomial \emph{random-iterate} convergence rate, measured by the expected duality gaps across all iterates. This result challenges the conventional wisdom that random-iterate convergence is essentially equivalent to best-iterate convergence, with the former often used as a proxy for establishing the latter.
Our analysis uncovers a new connection to dynamic regret and presents a novel two-phase approach to best-iterate convergence, which could be of independent interest.
\end{abstract}

\newpage
\section{Introduction}

No-regret learning dynamics provide one of the premier ways of computing equilibria in multiplayer interactions~\citep{cesa2006prediction}. They have been successfully deployed at scale across a wide variety of games and desired notions of equilibrium, and have been an integral part of superhuman AI for poker~\citep{brown2018superhuman,brown2019superhuman,moravvcik2017deepstack}, human-level AI for Stratego~\citep{perolat2022mastering} and Diplomacy~\citep{meta2022human}, as well as other uses such as alignment of large language models~\citep{munos2023nash,jacob2023consensus}.

In general, learning dynamics guarantee convergence to equilibrium in an \emph{ergodic} sense. In other words, it is not the actual behavior of the dynamics that converges to equilibrium, but rather their \emph{average} behavior.
Overcoming this limitation and establishing convergence \emph{in iterates} in the two-player zero-sum setting has been an important direction of research in the past decade. The reasons for this endeavor are multifaceted. First, dynamics that exhibit iterate convergence properties ensure that the learning agents will eventually play according to an equilibrium strategy. This is a desirable requirement for deploying learning in an online setting, where it is important that agents eventually sample actions from an optimal strategy. Second, algorithms with converging iterates rule out undesirable phenomena such as recurrence and even formally chaotic behavior~\citep{sato2002chaos,mertikopoulos2018cycles}. Finally, the construction of learning algorithms with iterate convergence guarantees is important for applications in nonconvex optimization where averaging the iterates is not possible~\citep{daskalakis2017training}.

In principle, at least three notions of iterate convergence can be identified, listed next from strongest to weakest.
\begin{itemize}
    \item \emph{Last-iterate} convergence, meaning that the learner's strategies approach the set of equilibrium strategies over time.
    \item \emph{Random-iterate} convergence, meaning that by sampling an iteration uniformly at random, the learner's strategy is close to equilibrium.
    \item \emph{Best-iterate} convergence, meaning that there exists a subsequence of strategies used by the learner that converges to the set of equilibrium strategies.
\end{itemize}
In these definitions, the proximity of the strategies to equilibrium is measured by the duality gap.
Orthogonal to the \emph{mode} of convergence above, is the \emph{speed} (\emph{i.e.}, non-asymptotic rate) of convergence, particularly regarding the dependence on possible \emph{condition numbers} of the payoff matrix of the two-player zero-sum game. In this regard, we identify two types of results.
\begin{itemize}
    \item \emph{Uniform} convergence results give a uniform upper bound on the convergence rate that applies to \emph{any} game instance. 
    \item \emph{Universal} convergence results also apply to \emph{any} game instance, but it can include a possibly arbitrarily bad dependence on some form of condition number of the game---for example, the smallest nonzero probability used in equilibrium.
\end{itemize}
Examples of universal and uniform convergence rates are given in Table \ref{tab:convergence rates}.
Clearly, universal convergence guarantees for learning dynamics are easier to establish than uniform convergence guarantees. In fact, most modern optimistic learning algorithms are known to enjoy some form of universal last-iterate convergence, albeit with a dependence on some form of condition numbers of the game, which can typically be arbitrarily large even for $2\times 2$ games~\citep{tseng1995linear,mordukhovich2010applying,wei2021linear}. The situation regarding uniform convergence results is significantly less understood. 

So far in the literature, there has never been an incentive to treat the three notions of \emph{uniform} iterate convergence separately. For example, while $O(T^{-\frac{1}{2}})$ uniform random-iterate convergence guarantees for Optimistic Gradient Descent Ascent (OGDA)~\citep{popov1980modification} have appeared earlier \citep{wei2021linear,anagnostides2022last}, they have been eventually strengthened to hold in the last iterate sense \citep{cai2022finite, gorbunov2022last}. For OGDA, the uniform convergence rates for last-, random-, and best-iterate are polynomial and essentially identical. Furthermore, prior to our work, no algorithm was known to exhibit a separation between these three convergence modes.

However, cracks in the above state of affairs have started to emerge recently. In an unexpected turn of events, a recent paper by \citet{cai2024fast} has shown that optimistic multiplicative weights update (OMWU)---a premier no-regret algorithm with otherwise best-in-class theoretical guarantees \citep{rakhlin2013optimization,daskalakis2021near-optimal,farina22:kernelized}---does not enjoy any {\em uniform} last-iterate convergence guarantees, despite the existing known {\em universal} last-iterate convergence rates~\citep{wei2021linear}. Crucially, that result only applies to uniform last-iterate convergence, and does not preclude the possibility that OMWU might still enjoy uniform random-iterate or best-iterate convergence guarantees.
This suggests the following question:
\begin{center}
    \itshape
    Does OMWU enjoy uniform random- or best-iterate convergence, despite the recently established negative result regarding its lack of uniform last-iterate convergence?
\end{center}
Answering this question is the main contribution of this paper.

\begin{table}[t]
\centering
\renewcommand{\arraystretch}{1.4}
\scalebox{.97}{\begin{tabular}{@{}p{2.8cm}|p{1.8cm}||p{1.8cm}|p{2.30cm}}

\bf \raisebox{-4mm}{Convergence}    & \bf OGDA  

(uniform) & \bf OMWU  

(uniform) & \bf OMWU 

(universal)\\ \toprule
Last iterate  & \multirow{3}{*}{\raisebox{-5mm}{\!\!$O(T^{-1/2})$}} & $\Omega(1)$  & 
$O\!\left(\exp\bigl(-\tfrac{T}{C}\bigr)C\right)^\dagger$ 
\\ \hhline{-|~|-|-} 
\vskip -3.5mm Random iterate &  & \cellcolor{black!10}\vskip-2mm\raisebox{3mm}{$\displaystyle\Omega\Big(\frac{1}{\log T}\Big)$} & \cellcolor{black!10}\raisebox{-3mm}{$O\!\left(T^{-\frac{1}{4}}\delta^{-\frac{1}{2}}\right)$} \\ 
\hhline{-|~|-|-} 
Best iterate &   & \multicolumn{2}{c}{\cellcolor{black!10}{$O(T^{-\frac{1}{6}})^{\ddagger}$}} \\ \bottomrule
\end{tabular}}
\vspace{-2mm}
\caption{\emph{Uniform} and {\em universal} convergence rates of OGDA and OMWU in zero-sum games with a fully mixed equilibrium.  Here $\delta > 0$ is the minimum probability in the Nash equilibrium. {\em Universal} convergence rates allow for dependence on $\delta$ while {\em uniform} convergence rates are $\delta$-independent (the constants only depend on the dimensions of the problems and the norm of the payoff matrix). Our new results are highlighted in gray. $^\dagger$: $C:= \Omega(\exp(\frac{1}{\delta}))$.
$\ddagger$: This upper bound only holds for $2 \times 2$ games.
}
\label{tab:convergence rates}
\end{table}

\subsection*{Contributions and Techniques}

In this paper, we show a \emph{separation} between the best-iterate and the random-iterate in terms of uniform convergence for OMWU in two-player zero-sum games, which also implies a separation between best- and last-iterate uniform convergence in the same setting. The separation is established by (i) a new lower bound for the uniform random-iterate convergence and (ii) a new upper bound for the uniform best-iterate convergences. Our results are summarized in \Cref{tab:convergence rates}.

\paragraph{Lower bounds} On the negative side, we show that OMWU does \emph{not} enjoy a polynomial uniform random-iterate convergence guarantee (\Cref{thm:random-lower-OMWU}). Our analysis also extends to the broader family of Optimistic Follow-The-Regularized-Leader (OFTRL) algorithms with the most popular regularizers and results in new lower bounds (\Cref{thm:random-lower-all-regularizers}). These lower bound results hold even for $2 \times 2$ games with a fully-mixed Nash equilibrium, and we give a numerical illustration in Figure \ref{fig:oftrl social regret}. 

\paragraph{Upper bounds} On the positive side, we prove that OMWU has an $O(T^{-\frac{1}{6}})$ uniform best-iterate convergence rate (\Cref{thm:best iterate OMWU}) for $2 \times 2$ games with a fully-mixed Nash equilibrium. We note that for the same class of games, OMWU has no uniform last-iterate convergence~\citep{cai2024fast} and no uniform polynomial random-iterate convergence, as we discussed above.
This result has some important consequences:
\begin{itemize}
    \item It partially counters the negative narrative on OMWU's convergence properties from \citep{cai2024fast}, by offering a positive, if slightly weaker, result that a uniform polynomial best-iterate convergence rate is possible. Extending our positive result beyond the 2-by-2 case is an interesting future direction.
    \item It shows, for the first time and on one of the most important algorithms in online learning, that uniform best-, random-, and last-iterate convergence properties need not go hand-in-hand, in contrast to all existing results in the literature, and thus that different techniques are necessary to study random-iterate and best-iterate convergence.
\end{itemize}
\paragraph{Techniques} As mentioned above, our positive result on the uniform polynomial best-iterate convergence rate of OMWU does not follow the common approach of showing uniform polynomial random-iterate convergence since our negative result precludes the latter. 
To sidestep this obstacle, we develop a novel \emph{two-phase approach} that we believe might be of independent interest for the study of best-iterate convergence properties beyond OMWU. 

In the global phase, we establish a global \emph{universal} random-iterate convergence rate $O(T^{-\frac{1}{4}}\delta^{-\frac{1}{2}})$ that has dependence on the minimum probability $\delta > 0$ among actions in the support of Nash equilibrium (\Cref{thm:OMWU dynamic}). To prove the result, we leverage the connection between random-iterate convergence, dynamic regret, and interval regret, which is new for proving random-iterate convergence. We remark that this result also holds for the general $d_1  \times d_2$ case and provides an exponential improvement to the best known bound on the last-iterate convergence rate of 
$O(\exp(\frac{1}{\delta})\cdot(1+\exp(-\frac{1}{\delta}))^{-T})$~\citep{wei2021linear}
in the dependence on the condition number $\delta$.

In the initial phase, we show that OMWU has fast \emph{uniform} convergence to one iterate with duality gap $O(\delta)$ (\Cref{thm:OMWU-initial}). Combining results in the two phases and the definition of the best-iterate convergence, we get a \emph{uniform} $\min\{\delta, T^{-\frac{1}{4}}\delta^{-\frac{1}{2}}\} = T^{-\frac{1}{6}}$ best-iterate convergence rate that is independent of $\delta$.

\section{Preliminaries}
    Let $\Delta^d \subseteq\-R^d$ be the $d$-dimension probability simplex. For a strictly convex regularizer $R: \+X \rightarrow \-R$, we denote its \emph{Bregman divergence} as $
        D_R(x,x') = R(x) - R(x') - \InAngles{\nabla R(x') , x - x'}.$

    We study online learning dynamics in a two-player zero-sum matrix game $\min_{x \in \Delta^{d_1}}\max_{\Delta^{d_2}} x^\top Ay$ with loss matrix $A \in [0,1]^{d_1 \times d_2}$. In each iteration $t \ge 1$,  the $x$-player chooses a mixed strategy $x^t \in \+X = \Delta^{d_1}$ and the $y$-player chooses a mixed strategy $y^t \in \+Y = \Delta^{d_2}$. Then the $x$-player receives loss vector $\ell^t_x = Ay^t$ while the $y$-player receives loss vector $\ell^t_y = -A^\top x^t$. The goal is convergence to a \emph{Nash equilibrium} $
    (x^*, y^*)$ where $ x^*\in \argmin_{x \in \+X}\max_{y \in \+Y} x^\top A y $ and $ y^*\in \argmax_{y \in \+Y}\min_{x \in \+X} x^\top A y $.
    A Nash equilibrium $(x^*, y^*)$ is \emph{fully mixed} if $x^*$ and $y^*$ both have full support.
    The proximity of a strategy profile $(x, y)$ to Nash equilibrium is measured by its \emph{duality gap}:
    \begin{align*}
        \DualGap(x,y) = \max_{y' \in \+Y} x^\top A y' - \min_{x' \in \+X} x'^\top Ay .
    \end{align*}
    The duality gap $\DualGap(x,y)$ is nonnegative and equals zero if and only if $(x,y)$ is a Nash equilibrium of the game $A$. 

    \paragraph{Online learning dynamics} We denote by $L_x^t = \sum_{k=1}^t \ell^k_x$ and $L^t_y = \sum_{k=1}^t \ell^k_y$ the cumulative loss vectors. The update rule of the Optimistic Follow-the-Regularized-Leader (OFTRL) algorithm~\citep{syrgkanis2015fast} with regularizer $R$ and step size $\eta > 0$ is: initialize $(x^1, y^1)$ both as the uniform distribution,\footnote{Throughout the paper, we assume all the algorithms are initialized with the uniform distribution.} then for each $t \ge 2$,
    \begin{equation}
    \label{OFTRL}
    \tag{OFTRL}
    \begin{aligned}
        x^t & = \argmin_{x \in \X} \left\{  \InAngles{x,  L^{t-1}_x + \ell^{t-1}_x } + \frac{1}{\eta} R(x)\right\}, \\
        y^t & = \argmin_{y \in \Y} \left\{  \InAngles{y, L^{t-1}_y + \ell^{t-1}_y } + \frac{1}{\eta} R(y)\right\}.
    \end{aligned}
    \end{equation}
    Commonly-studied regularizers include the following.
    \begin{itemize}
    \item Negative entropy, $R(x) \coloneqq \sum_{i=1}^{d}x[i]\log x[i]$. The resulting OFTRL algorithm is OMWU. 
    \item Squared Euclidean norm, $R(x) \coloneqq\frac{1}{2}\sum_{i=1}^dx[i]^2$.
    \item Log barrier, $R(x)\coloneqq\sum_{i=1}^d-\log (x[i])$, also called the ``log regularizer''.
    \item Negative Tsallis entropy family of regularizers, $R(x) \coloneqq ({1 - \sum_{i=1}^d (x[i])^\beta})/({1-\beta})$, parameterized by $\beta \in (0,1)$.
\end{itemize}

    Another popular family of online learning algorithms is the Optimistic Online Mirror Descent (OOMD) algorithm~\citep{rakhlin2013optimization}. We introduce some notations to simplify the presentation. We let $z = (x, y) \in \+Z := \Delta^{d_1} \times \Delta^{d_2}$ and define the gradient operator $F(z^t) = (\ell^t_x, \ell^t_y) = (Ay^t, -A^\top x^t)$. We also let $D_R(z,z') = D_R(x, x') + D_R(y,y')$. The update rule of OOMD is to initialize $z^1 = \hz^1$ and for all $t \ge 2$
    \begin{equation}
    \label{OOMD}
    \tag{OOMD}
    \begin{aligned}
        \hz^t &= \argmin_{z \in \+Z} \{ \eta \InAngles{z, F(z^{t-1})} + D_R(z, \hz^{t-1}) \} \\ 
        z^t &= \argmin_{z \in \+Z} \{ \eta \InAngles{z, F(z^{t-1})} + D_R(z, \hz^t) \} 
    \end{aligned}
    \end{equation}
    We remark that when $R(x) = \frac{1}{2}\InNorms{x}^2$, the resulting OOMD algorithm is also called the Optimistic Gradient Descent Ascent (OGDA) algorithm, and is different from the OFTRL algorithm instantiated with the same regularizer. However, for Legendre regularizers, including the entropy regularizer, the log regularizer, and the family of Tsallis entropy regularizers, their OFTRL and OOMD versions coincide in the sense that the iterates $\{z^t = (x^t, y^t)\}$ produced by OFTRL and OOMD are the same.

    This paper's main focus is the Optimistic Multiplicative Weights Update (OMWU) algorithm, though we also give several results for the broader class of OFTRL algorithms. OMWU is OFTRL/OOMD instantiated with the negative entropy regularizer $R = \sum_{i=1}^dx[i] \log x[i]$. OMWU admits the following closed-form update: 
    \begin{equation*}
    \label{OMWU}
    \begin{aligned}
        x^t[i] & \propto x^{t-1}[i] \cdot \exp\InParentheses{ -\eta L^{t-1}_x[i] - \eta \ell^{t-1}_x[i] }, i\in[d_1] \\
        y^t[j] & \propto y^{t-1}[j] \cdot \exp\InParentheses{ -\eta L^{t-1}_y[j] - \eta \ell^{t-1}_y[j] }, j\in[d_2].
    \end{aligned}
    \end{equation*}
    
    \paragraph{Notions of convergence} We focus on three types of convergence rates. For zero-sum games $A\in [0,1]^{d_1 \times d_2}$, we say an algorithm has a uniform last-iterate, random-iterate, or best-iterate convergence rate $O(f(T))$ (we omit dependence on $d_1,d_2$ here) if there exists a constant $c > 0$ such that for any game, any time $T \geq 1$, we have
    \begin{itemize}
        \item Last-iterate: $\DualGap(z^T) \le cf(T)$
        \item Random-iterate: $\mathbb{E}_{t \sim \mathrm{Uni}[1,T]}[\DualGap(z^t)] \le cf(T)$ 
        \item Best-iterate: $\min_{t \in[1,T]}[\DualGap(z^t)] \le cf(T)$
    \end{itemize}
    for $z^T:=(x^T, y^T).$
    By definition, a last-iterate convergence rate upper bounds the duality gap for every iterate $t \in [1, T]$; a random-iterate convergence rate upper bounds the average duality gap during time $[1, T]$, which is also the average \emph{social dynamic regret} (see \Cref{prop:dynamic regret->best-iterate}); and a best-iterate convergence rate upper bounds the smallest duality gap in time $[1, T]$. Clearly, last-iterate convergence is stronger than random-iterate convergence, and random-iterate convergence is stronger than best-iterate convergence. Yet, for OGDA, these three types of convergence admit similar rates (see Table \ref{tab:convergence rates}). Establishing a separation between these types of convergence for OMWU is the main focus of our paper.

\section{Lower Bounds for Random-Iterate Convergence}\label{sec:random-iterate lower bound}
In this section, we establish lower bounds for the random-iterate convergence rate of OFTRL dynamics in two-player zero-sum games. For OMWU, we present an impossibility result for a uniform polynomial convergence rate by establishing an $\Omega(\frac{1}{\ln T})$ lower bound, even in a simple class of $2 \times 2$ matrix games. 
\begin{theorem}\label{thm:random-lower-OMWU}
    For two-player zero-sum games with loss matrix $A\in [0,1]^{2 \times 2}$, the uniform random-iterate convergence rate of  OMWU with any constant step size $\eta\le \frac{1}{2}$ is $\Omega(\frac{1}{\log T})$.
    This result continues to hold if we restrict the space of loss matrices to games with fully-mixed Nash equilibria.
\end{theorem}
\begin{remark}
    We note that the lower bound holds for general $d_1 \times d_2$ games by a reduction given in \citet[Theorem 3]{cai2024fast}. \Cref{thm:random-lower-OMWU} has the following  implications:
    \begin{itemize}
        \item[1.] For any $T \ge 1$, we can find a game instance $A \in [0, 1]^{2 \times 2}$ such that OMWU on that game has nearly linear in $T$ social dynamic regret: $\sum_{t=1}^T \DualGap(x^t, y^t) = \Omega(\frac{T}{\log T})$.
        \item[2.] For any $\epsilon > 0$, we can find a game instance $A \in [0, 1]^{2 \times 2}$ such that OMWU on that game suffers $\frac{1}{T}\sum_{t=1}^T \DualGap(x^t, y^t)\ge \varepsilon$ even when $T = \Omega\left(\exp(\frac{1}{\varepsilon})\right)$.
    \end{itemize}
\end{remark}

Next we present lower bounds on the random-iterate convergence rate for OFTRL instantiated with the squared Euclidean norm, the log barrier, and the family of Tsallis entropy regularizers. 

\begin{theorem}\label{thm:random-lower-all-regularizers}
     For two-player zero-sum games $[0,1]^{2 \times 2}$, the following lower bounds hold for the uniform random-iterate convergence rate of OFTRL with constant step size:
     \begin{itemize}
        \item[1.] $\Omega(1)$ for the squared Euclidean norm regularizer
        \item[2.] $\Omega(T^{-\frac{1-\beta}{2-\beta}})$ for the Tsallis entropy regularizer parameterized by $\beta \in (0,1)$
        \item[3.] $\Omega(T^{-\frac{1}{2}})$ for the log regularizer
     \end{itemize}
     These results continue to hold if we restrict the space of loss matrices to games with fully-mixed Nash equilibria.
\end{theorem}
To the best of our knowledge, \Cref{thm:random-lower-OMWU} and \Cref{thm:random-lower-all-regularizers} are the first lower bound results for the random-iterate convergence rate of learning in games. Our lower bounds also offer insights into the relation between the random-iterate and the last-iterate convergence. By definition, we know that random-iterate convergence is a weaker requirement than last-iterate convergence. However, the $\Omega(1)$ lower bound for OFTRL with squared Euclidean norm shows that the uniform random-iterate convergence can still be arbitrarily slow. This matches the lower bound on its last-iterate convergence proved in~\citet{cai2024fast}, thereby demonstrating that random-iterate convergence can be \emph{as hard as} last-iterate convergence. 

\begin{figure*}
    \rotatebox{90}{\hskip5mm\small Avg. social dynamic regret}
    \includegraphics[width=.23\textwidth]{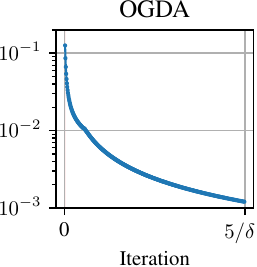}\hfill%
    \includegraphics[width=.23\textwidth]{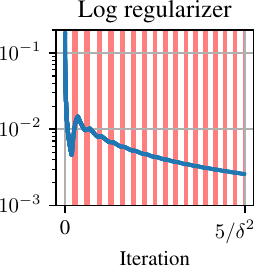}\hfill%
    \includegraphics[width=.23\textwidth]{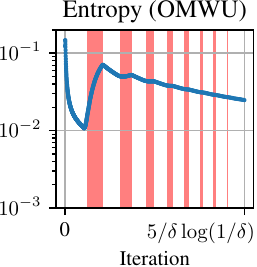}\hfill%
    \includegraphics[width=.23\textwidth]{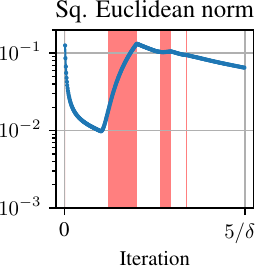}%
    \caption{Random-iterate convergence/average social dynamic regret guarantee of OGDA and OFTRL algorithms with log, entropy,  and squared Euclidean norm regularizer. The game is $A_\delta$ defined in \eqref{eq:A delta} with $\delta = 10^{-2}$. The red region is when the single iterate has a duality gap $\ge 0.1$. We intentionally show different numbers of iterations for different regularizers as illustrated in our lower bounds (\Cref{thm:random-lower-OMWU} and~\ref{thm:random-lower-all-regularizers}) and proofs (see \Cref{sec:random proof overview}).   
    \label{fig:oftrl social regret}}
\end{figure*}

\subsection{Proof Overview}\label{sec:random proof overview}
We present a high-level overview of the proofs of \Cref{thm:random-lower-OMWU} and \Cref{thm:random-lower-all-regularizers} here. The full proofs appear in \Cref{app:random}.  We focus on a class of hard instances introduced in~\citet{cai2024fast} that has a fully-mixed Nash equilibrium:
\begin{align}
    A_\delta := \begin{bmatrix}
        \frac{1}{2} + \delta & \frac{1}{2} \\
        0 &1
    \end{bmatrix}, \forall \delta \in \InParentheses{0,\frac{1}{2}}. \label{eq:A delta}
\end{align}
\citet{cai2024fast} establish an $\Omega(1)$ lower bound on the uniform last-iterate convergence rate for OFTRL dynamics with all the aforementioned regularizers. Specifically, they show that for any sufficiently-small $\delta > 0$, the OFTRL dynamics on $A_\delta$ always have at least one iterate $T = \Omega(1/\delta)$ such that $\DualGap(x^T, y^T) = \Omega(1)$. Therefore, a uniform $o(1)$ last-iterate convergence rate is impossible. 

Finding one iterate with a large duality gap is sufficient for the lower bound on the {\em last-iterate} convergence. However, that is not sufficient for the weaker notion of \emph{random-iterate} convergence, which measures the average duality gap:
\begin{align}
 \frac{1}{T} \sum_{t=1}^T \DualGap(x^t, y^t). \label{eq:random}
\end{align}
Instead, to show that \eqref{eq:random} is large, we must prove that a substantial proportion of iterations in $\{1,\ldots, T\}$ all have large duality gap. Building upon the analysis of the class of games~\eqref{eq:A delta} presented in~\citet{cai2024fast}, we  show that for some $T = \Omega(1/\delta)$ iterations, there will be a block of $\Theta(\frac{1}{\delta})$ iterations each with a constant duality gap: 
\begin{align*}
    \DualGap(x^t, y^t) = \Omega(1), \forall t \in \InBrackets{T, T+\Theta\InParentheses{\frac{1}{\delta}}}.
\end{align*}
As a result, the average duality gap at time $T + \Theta(\frac{1}{\delta})$ is at least
\begin{align*}
    \frac{1}{T+\Theta(\frac{1}{\delta})} \sum_{t=1}^{T+\Theta(\frac{1}{\delta})} \DualGap(x^t, y^t) \ge \frac{\Theta(\frac{1}{\delta})}{T+\Theta(\frac{1}{\delta})}.
\end{align*}
If we provide an upper bound on $T$ then we get a lower bound for the random-iterate convergence rate. Providing an upper bound on $T$ requires a careful analysis of the trajectory of the OFTRL learning dynamics and is presented as \Cref{theorem:consecutive bad iterations} in \Cref{app:random}.
The resulting upper bound is $T = O(\frac{1}{\delta}\cdot f_R(\delta))$ where $f_R(\delta)$ is a quantity that is related to the stability of the regularizer $R$. 
In \Cref{lemma: E delta} in \Cref{app:random}, we establish upper bounds for $f_R(\delta)$. For the entropy regularizer, we have $f_R(\delta) \le \log\frac{1}{\delta}$ which implies the $\Omega(\frac{1}{\log T})$ lower bound.  For the squared Euclidean norm, $f_R(\delta) \le 1$; for the log regularizer, $f_R(\delta) \le \frac{1}{\delta}$; for Tsallis entropy with $\beta \in (0,1)$, $f_R(\delta) \le \frac{2\beta}{1-\beta}(\frac{1}{\delta})^{1-\beta}$. See \Cref{fig:oftrl social regret} for a numerical illustration of the effects of different regularizers on the random-iterate convergence. Combining these results completes the proof for \Cref{thm:random-lower-OMWU} and ~\ref{thm:random-lower-all-regularizers}.

\section{Best-Iterate Convergence Rate for OMWU}\label{sec:best-iterate}
In this section, we establish a polynomial best-iterate convergence rate for OMWU for $2 \times 2$ matrix games with a fully mixed Nash equilibrium.  
Our main theorem in this section is the following.
\begin{theorem}\label{thm:best iterate OMWU}
     Consider any matrix game $A \in [0,1]^{2\times 2}$ with a fully mixed Nash equilibrium. Let $\{x^t,y^t\}_{t\ge 1}$ be the iterates produced by the OMWU dynamics with uniform initialization and constant step size $\eta \le \frac{1}{10}$. Then for any $T \ge 1$, we have
     \begin{align*}
         \min_{t \in [1,T]} \DualGap(x^t, y^t) = O(T^{-\frac{1}{6}}).
     \end{align*}
\end{theorem}
We note that all existing non-ergodic convergence rates of OMWU involve problem-dependent constants and can be arbitrarily slow~\citep{wei2021linear}. Our result is the first polynomial problem-independent best-iterate convergence rate for OMWU.
An important takeaway of \Cref{thm:best iterate OMWU} is that it provides, to our knowledge, the first separation between best-iterate convergence and random-iterate convergence for learning in games. 
For the same class of $2 \times 2$ games with a fully mixed Nash equilibrium, our results show that OMWU has an interesting landscape of convergence rates: (1) it does not have a uniform last-iterate convergence rate~\citep{cai2024fast}; (2) it does not have a polynomial uniform random-iterate convergence rate~(\Cref{thm:random-lower-OMWU}); (3) yet it has a polynomial $O(T^{-\frac{1}{6}})$ best-iterate convergence rate~(\Cref{thm:best iterate OMWU}). 
Our results show a surprising contrast with the uniform convergence results for OGDA, where the  last-iterate, best-iterate, and random-iterate convergence rates are similar.

\subsection{Proof Overview}\label{sec:overview-best}
To prove \Cref{thm:best iterate OMWU}, we develop a novel two-phase analysis of the best-iterate convergence rate of OMWU. In this section, we provide a high-level overview of the proof. A more detailed discussion of each phase appears in \Cref{sec:global phase} and \Cref{sec:initial phase}. The full proof is in \Cref{app:OMWU best-iterate}.

In the literature, the most common approach for proving a sublinear best-iterate convergence is using the random-iterate convergence as a proxy (since this is a stronger notion of convergence than best-iterate). Using OGDA as an example, it is known that the sum of duality gaps is sublinear $
    \sum_{t=1}^T \DualGap(x^t, y^t) = O(\sqrt{T})$.
This directly implies  $O(T^{-\frac{1}{2}})$ best-iterate convergence rate for OGDA.

However, this approach is impossible for OMWU, since we have shown a negative result for the random-iterate convergence rates of OMWU (Theorem \ref{thm:random-lower-OMWU}): for any $T$, there exists a game instance such that 
\begin{align*}
    \sum_{t=1}^T \DualGap(x^t, y^t) = \Omega\InParentheses{\frac{T}{\log T}}.
\end{align*}
As such, our proof of Theorem \ref{thm:best iterate OMWU} requires new insights fully tailored to the uniform best-iterate convergence (and independent of uniform random-iterate convergence).

We come up with the following two-phase analysis. Let $A$ be any $2\times 2$ game with a fully mixed Nash equilibrium and denote by $\delta > 0$ the minimum probability in the Nash equilibrium. 
\paragraph{Global phase} We first prove that for all $T$, we have a \emph{universal} $\delta$-dependent random-iterate convergence bound:
\begin{align*}
    \frac{1}{T}\sum_{t=1}^T \DualGap(x^t, y^t) = O(T^{-\frac{1}{4}}\delta^{-\frac{1}{2}}). 
\end{align*}
Note that this does not contradict \Cref{thm:random-lower-OMWU} since the bound has a dependence on $\delta$. It is worth noting that this bound holds for general $d_1\times d_2$ games with a fully mixed Nash equilibrium. The proof uses a connection between random-iterate convergence, dynamic regret, and interval regret. A detailed discussion is in \Cref{sec:global phase}.

\paragraph{Initial phase}  
We then analyze the initial iterations of OMWU. We show that there exists an iteration $T_1 \ge 1$ such that the following two conditions hold:
\begin{itemize}
    \item[1.] $\DualGap(x^{T_1}, y^{T_1}) = O(\delta)$.
    \item[2.] For all $T \in [1, T_1]$, a uniform best-iterate convergence rate holds: $\min_{t \in [1,T] }\DualGap(x^t,y^t) = \Tilde{O}(\frac{1}{T})$.
\end{itemize}
Here, $\tilde{O}(.)$ hides terms logarithmic in $T$. In summary, we show that the OMWU dynamics will reach an iterate $T_1$ with a duality gap of $O(\delta)$, and all the initial iterates $[1,T_1]$ have a fast best-iterate convergence rate independent of $\delta$. A detailed discussion is in \Cref{sec:initial phase}.

\paragraph{Combing the two-phase analysis} For all $T \le T_1$, by analysis in the initial phase, we know they have $\Tilde{O}(T^{-1})$
best-iterate convergence rate. For all $T \ge T_1$, we can combine the analysis in the initial phase and the global phase as follows:
\begin{align*}
    \min_{t \in [1,T]} \DualGap(x^t, y^t) &\le \min \{ \DualGap(x^{T_1}, y^{T_1}) , \frac{1}{T} \sum_{t=1}^T\DualGap(x^t,y^t)\} \\
    &\le \min \{ \delta , O(T^{-\frac{1}{4}}\delta^{-\frac{1}{2}})\} \le O(T^{-\frac{1}{6}}).
\end{align*}
Note that the last inequality holds since (1) if $\delta \le T^{-\frac{1}{6}}$, then the inequality holds; (2) if $\delta \ge T^{-\frac{1}{6}}$, then $T^{-\frac{1}{4}} \delta^{-\frac{1}{2}} \le T^{-\frac{1}{6}}$ and the inequality also holds. In this way, we get a uniform $\delta$-independent convergence rate.

The preceding analysis demonstrates the effectiveness of our two-phase approach, which leverages the additional flexibility inherent in best-iterate convergence compared to random-iterate convergence. We hope the insight in our approach will help analyze the best-iterate convergence rates of other algorithms. 

\subsection{Global Phase: Convergence via Minimizing the Social Dynamic Regret} \label{sec:global phase}
In this subsection, we present a \emph{universal} random-iterate convergence rate for OMWU on all matrix games $A \in [0,1]^{d_1 \times d_2}$ that have a fully mixed Nash equilibrium. Let $(x^*, y^*)$ be the fully mixed Nash equilibrium, we denote by $\delta =\min\{x^*[i], y^*[j]: i \in [d_1], j\in[d_2]\}$ as the minimum probability in the Nash equilibrium. We prove a random-iterate convergence rate of $O(T^{-\frac{1}{4}}\delta^{-\frac{1}{2}})$.

For simplicity of analysis, we use the OOMD-type update rule of OMWU in this section, which is equivalent to the OFTRL-type update of OMWU used in other parts of the paper. 
Recall that the OMWU algorithm initializes $z^1 = \hz^1$ as the uniform distribution and updates for iteration $t \ge 2$ with step size $\eta$:
\begin{align*}
    \hz^t &= \argmin_{z \in \+Z} \{ \eta \InAngles{z, F(z^{t-1})} + \KL(z, \hz^{t-1}) \} \\ 
    z^t &= \argmin_{z \in \+Z} \{ \eta \InAngles{z, F(z^{t-1})} + \KL(z, \hz^t) \} 
\end{align*}
Recall that $F(z^t) = (\ell^t_x, \ell^t_z) = (Ay^t, -A^\top x^t)$.

Our proof provides a new connection to the dynamic regret analysis from online learning.

\paragraph{Dynamic regret} Given loss functions $\{\ell^t_x\}$, actions $\{x^t\}$ produced by the $x$-player, and a sequence of comparators $\{u^t\}$, we define $x$-player's \emph{dynamic regret} as
\begin{align*}
    \+R^T_x(\{u^t\}):= \sum_{t=1}^T \InAngles{\ell^t_x, x^t - u^t}.
\end{align*}
A similar definition holds for the $y$-player. When $u^1=\ldots=u^t=u$, the dynamic regret recovers the standard static external regret. An interesting case is when the comparator sequence is optimal: $u^t = x^t_\star:= \min_{x \in \Delta^{d_1}} x^\top \ell^t_x$ for every $t \ge 1$, which is called the \emph{worst-case dynamic regret}. Observe that the sum of duality gaps is precisely the \emph{social dynamic regret}, \emph{i.e.,} the sum of both players' dynamic regret, as shown in \Cref{prop:dynamic regret->best-iterate}. 

\begin{proposition}\label{prop:dynamic regret->best-iterate}
    It holds that $\sum_{t=1}^T\DualGap(x^t,y^t)=\sum_{t=1}^T\max_{z \in \+Z} \InAngles{F(z^t), z^t-z} = \+R^T_x(\{x^t_\star\})+\+R^T_y(\{y^t_\star\})$. 
\end{proposition}

We then borrow insights from the online learning literature to upper bound the dynamic regret. The insight is that an algorithm has sublinear dynamic regret when (1) it has \emph{interval regret} guarantee $O(\sqrt{|\+I|})$ for all interval $\+I$ and (2) the {\em variation of loss sequence} is sublinear. We formally introduce interval regret and the variation of the loss sequence below and discuss how to prove these conditions for the OMWU dynamics. 

\paragraph{Interval Regret} An interval $\+I = [s,e] \subseteq[1,T]$ is a sequence of consecutive iterations. The interval regret with respect to interval $\+I$ is the standard regret but restricted to iterations in $\+I$:
\begin{align*}
    \+R^\+I:=\sum_{t\in\+I} (\ell^t)^\top x^t - \min_{x} \sum_{t \in \+I}(\ell^t)^\top x.
\end{align*}
Denote by $|\+I|$ the length of interval $\+I$. For online learning against adversarial losses, OMWU does not guarantee $\+R^\+I =o(|\+I|)$ for all $\+I$. However, as we show in the following, OMWU dynamics in zero-sum matrix games with a fully mixed Nash equilibrium guarantees the interval regret is constant (that depends on $\delta$): $\+R^\+I=O(1/\delta)$ for all $\+I$. 

We summarize some existing results of the OMWU dynamics in the following.
\begin{lemma}[Adapted from Lemma 1 in \citep{wei2021linear}]\label{lemma:OMWU rvu}
    Consider OWMU for a zero-sum game $A \in [0,1]^{d_1\times d_2}$ with $\eta \le \frac{1}{8}$. Let $z^*$ be a Nash equilibrium. Define $\Theta^t = \KL(z^*, \hz^t) + \frac{1}{16} \KL(\hz^t, z^{t-1})$ and $\zeta^t = \KL(\hz^{t+1}, z^t) + \KL(z^t, \hz^t)$. Then for any $t \ge 2$, we have
    \begin{itemize}
        \item[1.]  For any $z$, $\eta \InAngles{F(z^t), z^t-z}\le \Theta^t-\Theta^{t+1}-\frac{15}{16}\zeta^t$.
        \item[2.] $\Theta^{t+1}\le \Theta^t-\frac{15}{16}\zeta^t$.
        \item[3.] $\sum_{t=2}^\infty \InNorms{z^{t+1}-z^t}^2_1 = O(\log(d_1d_2))$.
    \end{itemize}
\end{lemma}

First, we note that since there is a fully mixed Nash equilibrium with minimum probability $\delta$, the minimum probability of any iterates is lower bounded by $\Omega(\exp(\frac{1}{\delta}))$. This is because by item 2 in \Cref{lemma:OMWU rvu}, we know $ \KL(z^*, \hz^t) \le \Theta^t \le \Theta^2$ is bounded by a constant that only depends on $d_1$ and $d_2$. Therefore, each coordinate of $\hz^t$ must be lower bounded by $\Omega(\frac{1}{\delta})$; otherwise the KL divergence $ \KL(z^*, \hz^t)$ would be too large. The guarantee for $z^t$ is by stability of the OMWU update. Formally, we have the next lemma.
\begin{lemma}[Adapted from Lemma 19 of \citep{wei2021linear}]\label{lemma:OMWU prob lower bound}
    Let matrix game $A \in [0,1]^{d_1\times d_2}$ has a fully mixed Nash equilibrium with minimum probability $\delta > 0$. Let $\{z^t\}_{t\ge1}$ be the iterates produced by OMWU with uniform initialization and step size $\eta \le \frac{1}{8}$, then 
    \begin{align*}
        &\min_{t \ge 1,i \in [d_1+d_2]}\{z^t[i], \hz^t[i]\} \ge \Omega\InParentheses{\frac{1}{d_1d_2}\exp\InParentheses{-\frac{1}{\delta}}}. \\
    \end{align*}
\end{lemma}

Combining \Cref{lemma:OMWU rvu} and \Cref{lemma:OMWU prob lower bound}, we can bound the social interval regret by $O(\frac{1}{\delta})$ for any interval $\+I = [s, e]$. Note that this bound, although has a dependence on $\delta$, is independent of the interval length $O(|\+I|)$ and thus is sufficient for our goal of achieving $O(\sqrt{|\+I|})$ bound.  The idea is that by item 1 in \Cref{lemma:OMWU rvu}, the sum of regret is bounded by $\Theta^s$, the sum of two KL divergences, whose upper bound follows by the probability lower bounds presented in \Cref{lemma:OMWU prob lower bound}.
\begin{lemma}[Bounded Interval Regret]\label{lemma:OMWU interval} In the same setup as \Cref{lemma:OMWU prob lower bound}, we have $
         \+R^\+I_x+\+R^\+I_y = O\InParentheses{\frac{(d_1+d_2)\log(d_1d_2)}{\eta \delta}}, \forall\+I.
     $
\end{lemma}

\paragraph{Sublinear variation of loss sequences}
Next, we show that the OMWU dynamic produces a \emph{stable} environment, \emph{i.e.,} the variation of loss functions is small. Formally, we define the variation over an interval $\+I = [s,e]$ as
\begin{align*}
    V^\+I = \sum_{t=s+1}^e \max_{z} |\InAngles{F(z^t) - F(z^{t-1}),z}| . 
\end{align*}
We note that OMWU has bounded second-order path length $\sum_{t=2}^\infty\InNorms{{z^t -z^{t-1}}}^2_1$ (by item 3 in \Cref{lemma:OMWU rvu}). We then have the following bound on its variation of losses.
\begin{lemma}\label{lemma:OMWU variation}
    For any $\+I$, we have $ V^\+I =  O\InParentheses{\sqrt{|\+I|\log(d_1d_2)}}$.
\end{lemma}

\paragraph{From interval regret to dynamic regret} Now we combine the bound for the interval regret (\Cref{lemma:OMWU interval}) and the sublinear variation of the loss sequence bound (\Cref{lemma:OMWU variation}) to show a sublinear dynamic regret bound. 
\begin{theorem}\label{thm:OMWU dynamic}
     Let matrix game $A \in [0,1]^{d_1\times d_2}$ have a fully mixed Nash equilibrium with minimum probability $\delta > 0$. Let $\{x^t,y^t\}_{t\ge1}$ be the iterates produced by OMWU with uniform initialization and step size $\eta \le \frac{1}{8}$. Then for any $T \ge 2$, the social dynamic regret is bounded by
     \begin{align*}
         \+R^T_x(x^t_\star)+\+R^T_y(y^t_\star) = O\InParentheses{ \frac{(d_1+d_2)\log(d_1d_2)}{\eta} \cdot T^{\frac{3}{4}}\delta^{-\frac{1}{2}}}.
     \end{align*}
\end{theorem}
We remark that the rate $O(T^{\frac{3}{4}}\delta^{-\frac{1}{2}})$ improves the existing bound $O(\exp(\frac{1}{\delta})\cdot(1+\exp(-\frac{1}{\delta}))^{-T})$ exponentially in terms of the dependence on $\delta$. We give a proof sketch of \Cref{thm:OMWU dynamic} here and the full proof is in \Cref{app:global phase}. Recall that $z^t_\star$ is the optimal comparator in time $t$. The idea is that the dynamic regret for an interval $\+I = [s, e]$ could be decomposed as follows:
\begin{align*}
    \sum_{t\in \+I} \InAngles{F(z^t), z^t - z^t_\star} = \sum_{t\in \+I} \InAngles{F(z^t), z^t -z^{s}_\star} + \sum_{t\in \+I} \InAngles{F(z^t),z^{s}_\star - z^t_\star}
\end{align*}
where the first term is bounded by the interval regret $O(\frac{1}{\delta})$, and the second term could be bounded by the variation of the loss sequence. Then we can choose an optimal partition of $T$ rounds that achieves the optimal rate.

\subsection{Initial Phase: Fast Convergence towards $O(\delta)$} \label{sec:initial phase}
In this section, we return to the case of games $[0,1]^{2\times2}$ with a fully mixed Nash equilibrium and present a fast uniform best-iterate convergence rate of OMWU for the initial iterations. Recall that our goal is to show that there exists an iteration $T_1 \ge 1$, the following two conditions hold:
\begin{itemize}
    \item[1.] $\DualGap(x^{T_1}, y^{T_1}) = O(\delta)$.
    \item[2.] For all $T \in [1, T_1]$, a uniform best-iterate convergence rate holds: $\min_{t \in [1,T] }\DualGap(x^t,y^t) = \Tilde{O}(\frac{1}{T})$.
\end{itemize}

\paragraph{Simplification by structure} Consider the game $A \in [0,1]^{2 \times 2}$ that has a fully mixed Nash equilibrium $(x^* = (1 - \delta_x, \delta_x), y^*= ( 1 - \delta_y, \delta_y))$ with $\delta_x, \delta_y \in (0,1)$. To simplify the analysis, we make the following assumption.
\begin{assumption}\label{assumption:A2}
    The Nash equilibrium of $A$ satisfies: $0 < \delta_x \le \delta_y \le 1 - \delta_x$ and $\delta_x \le \frac{1}{100}$.
\end{assumption}
We remark that \Cref{assumption:A2} is without loss of generality. The assumption $\delta_x \le 1-\delta_x$ (thus $\delta_x \le \frac{1}{2}$) holds without loss of generality since we can exchange the role of action $1$ and action $2$; the assumption $x^*$ is closer to the boundary than $y^*$, \emph{i.e.,} $\delta_x \le \delta_y\le 1-\delta_x$, holds without loss of generality since we can exchange the role of $x$-player and $y$-player. With \Cref{assumption:A2}, we let $\delta:=\delta_x$ be the minimum probability in the Nash equilibrium. The assumption $\delta_x\le \frac{1}{100}$ is without loss of generality since otherwise \Cref{thm:OMWU dynamic} already gives a $O(T^{-\frac{1}{4}}\delta^{-\frac{1}{2}}) = O(T^{-\frac{1}{4}})$ rate. We focus on the more challenging case when $\delta_x$ could be arbitrarily close to $0$ here.

We further make an important observation on the structure of $2 \times 2$ games. We show that every game matrix with a fully mixed Nash equilibrium is an affine transformation of a generic matrix. Formally, we have the following lemma.
\begin{lemma}\label{lemma:structure of 2by2 game}
    Let $A \in \-R^{2\times 2}$ be a matrix game such that it has a Nash equilibrium $(x^* = (1 - \delta_x, \delta_x), y^*= ( 1 - \delta_y, \delta_y)])$ with $\delta_x, \delta_y\in (0,1)$. Then $A$ can be written as 
    \begin{align*}
        A = b_1 \cdot \textbf{1} + b_2 \cdot \underbrace{\begin{bmatrix}
        \frac{1-\delta_y}{1-\delta_x}  &\frac{1- \delta_x- \delta_y}{1-\delta_x}\\
        0 & 1
    \end{bmatrix}}_{A_{\delta_x,\delta_y}},
    \end{align*}
    where $\textbf{1}$ is the all-ones matrix and $b_1,b_2 \in \mathbb{R}, b_2 \ge 0$ are scaling constants.
\end{lemma}

Together with \Cref{assumption:A2}, we have $A_{\delta_x, \delta_y} \in [0,1]^{2 \times 2}$.
 
Considering the constraint that $A \in [0,1]^{2 \times 2}$, we can assume $b_2 \in(0,1]$ (understanding that if $b_2 = 0$, then the matrix is all-zero and every point is a Nash equilibrium). The next proposition further shows that we only need to prove convergence for the case $b_1 =0, b_2 = 1$.

\begin{proposition}\label{prop:reduction}
    Let $A = b_1 \cdot \textbf{1} + b_2 A_{\delta_x, \delta_y}$ where $b_2 \in (0,1]$ and $\eta'>0$ be a constant. If OMWU with any step size $0 <\eta \le \eta'$ has convergence rate $O(\frac{1}{\eta} \cdot f(T))$ on $A_{\delta_x, \delta_y}$, then OMWU with any step size $0 <\eta \le \eta'$ also has convergence rate $O(\frac{1}{\eta} \cdot f(T))$ on $A$.
\end{proposition}

\paragraph{Fast convergence on $A_{\delta_x, \delta_y}$} By previous simplification that is without loss of generality, we now focus on the class of games $A_{\delta_x, \delta_y}$. Formally, we have the following theorem.
\begin{theorem}[Fast Convergence in Initial Iterations]\label{thm:OMWU-initial}
    Consider matrix game $A_{\delta_x, \delta_y}$ that satisfies \Cref{assumption:A2}. Let $\{x^t,y^t\}_{t\ge 1}$ be the iterates produced by OMWU dynamics with uniform initialization and step size $\eta \le \frac{1}{10}$. Then there exists $T_1 > 1$ such that 
    \begin{itemize}
        \item[1.] $\DualGap(x^{T_1}, y^{T_1}) \le 2\delta_x$. 
        \item[2.] For all $t \in [1, T_1]$, we have problem-independent best-iterate convergence rate: $$\min_{k\in[1,t]} \DualGap(x^k,y^k) = O\InParentheses{\frac{\log^2t}{\eta t}}.$$
    \end{itemize}
\end{theorem}
Thanks to the structure of $A_{\delta_x,\delta_y}$, we can directly track the initial trajectory of OMWU dynamics. Specifically, we define $T_1$ as the first iteration when $x^{T_1}[1] \ge 1-\delta_x$ (note that $x^1[1] = \frac{1}{2}$) and show that the initial iterations $[1, T_1]$ has the desirable properties. The full proof is somewhat involved and appears in \Cref{sec:OMWU best-initial}.

\section{Conclusion and Discussion}
In this paper, we establish, for the first time, a separation between random-iterate and best-iterate convergence: OMWU has no polynomial random-iterate convergence for two-player zero-sum games but has a $O(T^{-\frac{1}{6}})$ best-iterate convergence rate for the case of $2 \times 2$ games with fully mixed Nash equilibria. Whether obtaining polynomial best-iterate convergence for OMWU in the general case is possible is an interesting open question for future works.

\subsection*{Acknowledgement}
Yang Cai was supported by the NSF Awards CCF-1942583 (CAREER) and CCF-2342642. Gabriele Farina was supported by the NSF Award CCF-2443068 (CAREER). Julien Grand-Cl{\'e}ment was supported by Hi! Paris and
Agence Nationale de la Recherche (Grant 11-LABX-0047). Christian Kroer was supported by the Office of Naval Research awards N00014-22-1-2530 and N00014-23-1-2374, and the National Science Foundation awards IIS-2147361 and IIS-2238960.
Haipeng Luo was supported by the National Science Foundation award IIS-1943607. Weiqiang Zheng was supported by the NSF Awards CCF-1942583 (CAREER), CCF-2342642, and a Research Fellowship from the Center for Algorithms, Data, and Market Design at Yale (CADMY).

\bibliography{ref}

\begin{thebibliography}{23}
\providecommand{\natexlab}[1]{#1}
\providecommand{\url}[1]{\texttt{#1}}
\expandafter\ifx\csname urlstyle\endcsname\relax
  \providecommand{\doi}[1]{doi: #1}\else
  \providecommand{\doi}{doi: \begingroup \urlstyle{rm}\Url}\fi

\bibitem[Anagnostides et~al.(2022)Anagnostides, Panageas, Farina, and Sandholm]{anagnostides2022last}
Anagnostides, I., Panageas, I., Farina, G., and Sandholm, T.
\newblock On last-iterate convergence beyond zero-sum games.
\newblock In \emph{International Conference on Machine Learning}, pp.\  536--581. PMLR, 2022.

\bibitem[Brown \& Sandholm(2018)Brown and Sandholm]{brown2018superhuman}
Brown, N. and Sandholm, T.
\newblock Superhuman {AI} for heads-up no-limit poker: Libratus beats top professionals.
\newblock \emph{Science}, 359\penalty0 (6374):\penalty0 418--424, 2018.

\bibitem[Brown \& Sandholm(2019)Brown and Sandholm]{brown2019superhuman}
Brown, N. and Sandholm, T.
\newblock Superhuman ai for multiplayer poker.
\newblock \emph{Science}, 365\penalty0 (6456):\penalty0 885--890, 2019.

\bibitem[Cai et~al.(2022)Cai, Oikonomou, and Zheng]{cai2022finite}
Cai, Y., Oikonomou, A., and Zheng, W.
\newblock Finite-time last-iterate convergence for learning in multi-player games.
\newblock In \emph{Advances in Neural Information Processing Systems (NeurIPS)}, 2022.

\bibitem[Cai et~al.(2024)Cai, Farina, Grand-Cl{\'e}ment, Kroer, Lee, Luo, and Zheng]{cai2024fast}
Cai, Y., Farina, G., Grand-Cl{\'e}ment, J., Kroer, C., Lee, C.-W., Luo, H., and Zheng, W.
\newblock Fast last-iterate convergence of learning in games requires forgetful algorithms.
\newblock In \emph{The Thirty-eighth Annual Conference on Neural Information Processing Systems}, 2024.
\newblock URL \url{https://openreview.net/forum?id=hK7XTpCtBi}.

\bibitem[Cesa-Bianchi \& Lugosi(2006)Cesa-Bianchi and Lugosi]{cesa2006prediction}
Cesa-Bianchi, N. and Lugosi, G.
\newblock \emph{Prediction, learning, and games}.
\newblock Cambridge university press, 2006.

\bibitem[Daskalakis et~al.(2017)Daskalakis, Ilyas, Syrgkanis, and Zeng]{daskalakis2017training}
Daskalakis, C., Ilyas, A., Syrgkanis, V., and Zeng, H.
\newblock Training gans with optimism.
\newblock \emph{arXiv preprint arXiv:1711.00141}, 2017.

\bibitem[Daskalakis et~al.(2021)Daskalakis, Fishelson, and Golowich]{daskalakis2021near-optimal}
Daskalakis, C., Fishelson, M., and Golowich, N.
\newblock Near-optimal no-regret learning in general games.
\newblock \emph{Advances in Neural Information Processing Systems (NeurIPS)}, 2021.

\bibitem[Farina et~al.(2022)Farina, Lee, Luo, and Kroer]{farina22:kernelized}
Farina, G., Lee, C.-W., Luo, H., and Kroer, C.
\newblock Kernelized multiplicative weights for 0/1-polyhedral games: Bridging the gap between learning in extensive-form and normal-form games.
\newblock In \emph{International Conference on Machine Learning (ICML)}, pp.\  6337--6357, 2022.

\bibitem[Gorbunov et~al.(2022)Gorbunov, Taylor, and Gidel]{gorbunov2022last}
Gorbunov, E., Taylor, A., and Gidel, G.
\newblock Last-iterate convergence of optimistic gradient method for monotone variational inequalities.
\newblock In \emph{Advances in Neural Information Processing Systems}, 2022.

\bibitem[Jacob et~al.(2023)Jacob, Shen, Farina, and Andreas]{jacob2023consensus}
Jacob, A.~P., Shen, Y., Farina, G., and Andreas, J.
\newblock The consensus game: Language model generation via equilibrium search.
\newblock \emph{arXiv preprint arXiv:2310.09139}, 2023.

\bibitem[Mertikopoulos et~al.(2018)Mertikopoulos, Papadimitriou, and Piliouras]{mertikopoulos2018cycles}
Mertikopoulos, P., Papadimitriou, C., and Piliouras, G.
\newblock Cycles in adversarial regularized learning.
\newblock In \emph{Proceedings of the twenty-ninth annual ACM-SIAM symposium on discrete algorithms}, pp.\  2703--2717. SIAM, 2018.

\bibitem[Meta et~al.(2022)Meta, Bakhtin, Brown, Dinan, Farina, Flaherty, Fried, Goff, Gray, Hu, et~al.]{meta2022human}
Meta, Bakhtin, A., Brown, N., Dinan, E., Farina, G., Flaherty, C., Fried, D., Goff, A., Gray, J., Hu, H., et~al.
\newblock Human-level play in the game of diplomacy by combining language models with strategic reasoning.
\newblock \emph{Science}, 378\penalty0 (6624):\penalty0 1067--1074, 2022.

\bibitem[Morav{\v{c}}{\'\i}k et~al.(2017)Morav{\v{c}}{\'\i}k, Schmid, Burch, Lis{\`y}, Morrill, Bard, Davis, Waugh, Johanson, and Bowling]{moravvcik2017deepstack}
Morav{\v{c}}{\'\i}k, M., Schmid, M., Burch, N., Lis{\`y}, V., Morrill, D., Bard, N., Davis, T., Waugh, K., Johanson, M., and Bowling, M.
\newblock Deepstack: Expert-level artificial intelligence in heads-up no-limit poker.
\newblock \emph{Science}, 356\penalty0 (6337):\penalty0 508--513, 2017.

\bibitem[Mordukhovich et~al.(2010)Mordukhovich, Pena, and Roshchina]{mordukhovich2010applying}
Mordukhovich, B.~S., Pena, J.~F., and Roshchina, V.
\newblock Applying metric regularity to compute a condition measure of a smoothing algorithm for matrix games.
\newblock \emph{SIAM Journal on Optimization}, 20\penalty0 (6):\penalty0 3490--3511, 2010.

\bibitem[Munos et~al.(2023)Munos, Valko, Calandriello, Azar, Rowland, Guo, Tang, Geist, Mesnard, Michi, et~al.]{munos2023nash}
Munos, R., Valko, M., Calandriello, D., Azar, M.~G., Rowland, M., Guo, Z.~D., Tang, Y., Geist, M., Mesnard, T., Michi, A., et~al.
\newblock Nash learning from human feedback.
\newblock \emph{arXiv preprint arXiv:2312.00886}, 2023.

\bibitem[Perolat et~al.(2022)Perolat, De~Vylder, Hennes, Tarassov, Strub, de~Boer, Muller, Connor, Burch, Anthony, et~al.]{perolat2022mastering}
Perolat, J., De~Vylder, B., Hennes, D., Tarassov, E., Strub, F., de~Boer, V., Muller, P., Connor, J.~T., Burch, N., Anthony, T., et~al.
\newblock Mastering the game of stratego with model-free multiagent reinforcement learning.
\newblock \emph{Science}, 378\penalty0 (6623):\penalty0 990--996, 2022.

\bibitem[Popov(1980)]{popov1980modification}
Popov, L.~D.
\newblock A modification of the arrow-hurwicz method for search of saddle points.
\newblock \emph{Mathematical notes of the Academy of Sciences of the USSR}, 28:\penalty0 845--848, 1980.

\bibitem[Rakhlin \& Sridharan(2013)Rakhlin and Sridharan]{rakhlin2013optimization}
Rakhlin, S. and Sridharan, K.
\newblock Optimization, learning, and games with predictable sequences.
\newblock \emph{Advances in Neural Information Processing Systems}, 2013.

\bibitem[Sato et~al.(2002)Sato, Akiyama, and Farmer]{sato2002chaos}
Sato, Y., Akiyama, E., and Farmer, J.~D.
\newblock Chaos in learning a simple two-person game.
\newblock \emph{Proceedings of the National Academy of Sciences}, 99\penalty0 (7):\penalty0 4748--4751, 2002.

\bibitem[Syrgkanis et~al.(2015)Syrgkanis, Agarwal, Luo, and Schapire]{syrgkanis2015fast}
Syrgkanis, V., Agarwal, A., Luo, H., and Schapire, R.~E.
\newblock Fast convergence of regularized learning in games.
\newblock \emph{Advances in Neural Information Processing Systems (NeurIPS)}, 2015.

\bibitem[Tseng(1995)]{tseng1995linear}
Tseng, P.
\newblock On linear convergence of iterative methods for the variational inequality problem.
\newblock \emph{Journal of Computational and Applied Mathematics}, 60\penalty0 (1-2):\penalty0 237--252, 1995.

\bibitem[Wei et~al.(2021)Wei, Lee, Zhang, and Luo]{wei2021linear}
Wei, C.-Y., Lee, C.-W., Zhang, M., and Luo, H.
\newblock Linear last-iterate convergence in constrained saddle-point optimization.
\newblock In \emph{International Conference on Learning Representations (ICLR)}, 2021.

\end{thebibliography}
\bibliographystyle{icml2025}

\newpage
\appendix
\onecolumn

\tableofcontents


\section{Missing Proofs in \Cref{sec:random-iterate lower bound}}\label{app:random}
In this section, we prove lower bounds for the random-iterate convergence rates of OMWU and other OFTRL dynamics. Recall that we focus on the specific class of $2 \times 2$ zero-sum matrices parameterized by $\delta \in (0,\frac{1}{2})$:
\begin{align*}
    A_\delta := \begin{bmatrix}
        \frac{1}{2} + \delta & \frac{1}{2} \\
        0 &1
    \end{bmatrix}, \forall \delta \in \InParentheses{0,\frac{1}{2}}. 
\end{align*}
which is the hard instance for last-iterate convergence analyzed in~\citep{cai2024fast}. We first introduce some notations and present the results and analysis in~\citep{cai2024fast} that will be useful in our later analysis in \Cref{sec:old results}. Then, we give the proof of the lower bound for random-iterate convergence in \Cref{app:random-iterate lower bound}. 
\subsection{Existing Analysis on the Hard Instance}\label{sec:old results}
\paragraph{Additional Notations} We focus on the $d=2$ dimension case where $x = (x[1], y[1])$. We let $e^t_x = \ell^t_x[1] - \ell^t_x[2]$ be the difference between the losses of action $1$ and action $2$, and $E^t_x= \sum_{k=1}^t e^t_x$ be the cumulative differences. Define the function $F_{\eta, R} : \-R \rightarrow [0,1]$ as follows:
\begin{align}
    \label{function:F}
    F_{\eta, R}(e) := \argmin_{x \in [0,1]}\left\{ x \cdot e + \frac{1}{\eta} R(x) \right\}.
\end{align}
where we slightly abuse the notation and write $R((x,1-x))$ as $R(x)$, where $R$ is a strongly convex regularize over $\Delta^d$. Then by a change of variable $x^t[2] = 1 - x^t[1]$ the optimization problem in the OFTRL update can be reduced to a $1$-dimension optimization problem as follows:
\begin{align*}
    x^t[1] = F_{\eta, R}(E^{t-1}_x + e^{t-1}_x), \quad x^t[2] = 1 - x^t[1].
\end{align*}

\begin{lemma}[Lemma 1 in~\citep{cai2024fast}]
    \label{lemma: F monotone}
    The function $F_{\eta, R}(\cdot): \-R \rightarrow [0,1]$ defined in \eqref{function:F} is non-increasing.
\end{lemma}

\citet{cai2024fast} introduce the following assumptions on the regularizer $R$ and show that they are satisfied by the negative entropy, the log regularizer, the squared Euclidean norm, and the Negative Tsallis entropy regularizers.
\begin{assumption}
    \label{assumption:standard}
    We assume that the regularizer $R$ satisfies the following properties: the function $F_{\eta, R}: \-R \rightarrow [0,1]$ defined in \eqref{function:F} is: 
    \begin{itemize}
        \item[1.] \textbf{Unbiased:} $F_{\eta, R}(0) = \frac{1}{2}$.
        \item[2.] \textbf{Rational:} $\lim_{E \rightarrow -\infty}F_{\eta, R}(E) = 1$ and $\lim_{E \rightarrow +\infty}F_{\eta, R}(E) = 0$.
        \item[3.] \textbf{Lipschitz continuous:} There exists $L \ge 0$ such that $F_{1, R}$ is $L$-Lipschitz.
    \end{itemize}
\end{assumption}

\begin{assumption}
\label{assumption:main}
    Let $L$ be the Lipschitness constant of $F_{1,R}$ in \Cref{assumption:standard}. Define the constant $c_1 = \frac{1}{2} - F_{1, R}(\frac{1}{20L})$. There exist universal constants $\delta', c_2 > 0$ and $c_3 \in (0,\frac{1}{2}]$ such that for any $ 0 < \delta \le \delta'$,
    \begin{itemize}
        \item[1.] 
            For any $E$ that satisfies $F_{1, R}(E) \ge \frac{1}{1+\delta}$, we have $F_{1, R}(-\frac{c_1^2}{30 L\delta} + E) \ge \frac{1 + c_3}{1 + c_3 + \delta}$
        \item[2.] 
             For any $E$ that satisfies $F_{1, R}(E) \ge \frac{1}{2(1+\delta)}$, we have $F_{1, R}(-\frac{c_3 c_1^2}{120L} + \frac{\delta}{4L} + E) \ge \frac{1}{2} + c_2$.
    \end{itemize}
\end{assumption}
\begin{lemma}[Lemma 5-8 in \citep{cai2024fast}]\label{lemma:assumptions}
    \Cref{assumption:standard} and \Cref{assumption:main} are satisfied by negative entropy ($L=\frac{1}{2}$), squared Euclidean norm ($L=\frac{1}{2}$), the log regularizer ($L=\frac{1}{2}$), and the Tsallis entropy regularizer parameterized with $\beta \in (0,1)$ ( $L = \frac{1}{2\beta}$ ).
\end{lemma}

We summarize the analysis from~\citep{cai2024fast} in the following theorem.
\begin{theorem}[Adapted from Theorem 1 in \citep{cai2024fast}]
    \label{theorem: main}
    Assume the regularizer $R$ satisfies \Cref{assumption:standard} and \Cref{assumption:main}. For any $\delta \in (0,\hat{\delta})$, where $\hat{\delta}=\min\{\frac{1}{15}, \frac{c_1}{6}, \frac{c_1^2}{300}, \delta'\}$ is a constant depending only on the constants $c_1$ and $\delta'$ defined in \Cref{assumption:main}, let $\{x^t,y^t\}$ be iterates produced by the OFTRL dynamics on $A_\delta$ (defined in \eqref{eq:A delta}) with any step size $\eta \le \frac{1}{4L}$ and initialized at the uniform strategies.  Then the following holds:
    \begin{itemize}
        \item[1.] Define $T_1$ the simplest iteration when $x^{T_1}[1] \ge \frac{1}{1+\delta}$. Then \[y^t[1] \le \frac{1}{2}-c_1, \forall t \in \InBrackets{\frac{1}{2\eta L}, T_1 -1}.\]
        \item[2.] Define $T_2 > T_1$ the smallest iteration when $y^{T_2}[1] \ge \frac{1}{2(1+\delta)}$ and $T_h := \lfloor \frac{c_1}{2\eta L \delta} \rfloor \in [\frac{c_1}{2\eta L \delta}-1, \frac{c_1}{2\eta L \delta}]$. Then
        \begin{align*}
            x^t[1] \ge \frac{1+c_3}{1+c_3+\delta}, \quad &\forall t \in [T_1+ T_h, T_2] \\
            \DualGap(x^t, y^t) \ge c_2, \quad &\forall t \in \InBrackets{T_2 + \lceil \frac{c_1T_h}{20}\rceil, T_2 + \lfloor \frac{c_1T_h}{10} \rfloor-2}
        \end{align*}
    \end{itemize}
\end{theorem}

\subsection{Lower Bounds for Random-Iterate}\label{app:random-iterate lower bound}

In this section, we present our lower bounds for random-iterate convergence rates of OFTRL learning dynamics. The main theorem in this section is \Cref{theorem:consecutive bad iterations} that shows there exists $T = O(\frac{f_R(\delta)}{\delta})$ such that all the iterates $t \in [T, T+\Omega(\frac{1}{\delta})]$ have a duality gap lower bounded by a constant. Here $f_R(\delta):=-F_{1,R}^{-1}\InParentheses{\frac{1}{1+\delta}}$ is a quantity that depends on the regularizer $R$. We provide upper bounds of $f_R(\delta)$ in \Cref{lemma: E delta}. At the end, we combine \Cref{theorem:consecutive bad iterations} and \Cref{lemma: E delta} to prove \Cref{thm:random-lower-OMWU} and \Cref{thm:random-lower-all-regularizers}.

\begin{lemma}\label{lemma: E delta}
   Let $f_R(\delta):=-F_{1,R}^{-1}\InParentheses{\frac{1}{1+\delta}}$. If $E \ge -f_R(\delta)$, then $F_{1,R}(E) \le \frac{1}{1+\delta}$. Moreover, we have the following upper bounds on $f_R(\delta)$ for $\delta \in (0,\frac{1}{2})$ and regularizer $R$:
    \begin{itemize}
        \item[1.] Negative entropy: $f_R(\delta) \le \log(\frac{1}{\delta})$;
        \item[2.] Squared Euclidean norm: $f_R(\delta) \le 1$.
        \item[3.] Log barrier: $f_R(\delta) \le \frac{1}{\delta}$.
        \item[4.] Negative Tsallis entropy ($\beta \in (0,1)$): $f_R(\delta) \le \frac{2\beta}{1-\beta}(\frac{1}{\delta})^{1-\beta}$.
    \end{itemize}
\end{lemma}
\begin{proof}
    Since $f_R(\delta) = - F^{-1}_{1,R}\InParentheses{\frac{1}{1+\delta}}$, we have
    \begin{align*}
        F_{1,R}\InParentheses{-f_R(\delta)} = F_{1,R} \InParentheses{ F^{-1}_{1,R}\InParentheses{\frac{1}{1+\delta}}} = \frac{1}{1+\delta}. 
    \end{align*}
    Since $F_{1, R}$ is a non-increasing function (\Cref{lemma: F monotone}), we know if $E \ge -f_R(\delta)$, then $F_{1,R}(E) \le \frac{1}{1+\delta}$. This finishes the proof of the first claim. 

    We then prove upper bounds for $f_R(\delta)$ for different regularizer $R$. Our proof strategy is to derive the closed-form expression of $F_{1,R}$ (by setting the gradient of \Cref{function:F} to be $0$) and that of its inverse function $F^{-1}_{1,R}$ and directly bound $f_R(\delta)$.
    \paragraph{Negative entropy: $R(x) = x \log x + (1-x) \log (1-x)$} For negative entropy, $F_{1,R}$ admits the following closed-form expression:
    \begin{align*}
        F_{1,R}[E] = \frac{1}{1+\exp(E)} \Rightarrow F_{1,R}^{-1}\InParentheses{\frac{1}{1+\delta}} = \log\delta
    \end{align*}
    Hence, $f_R(\delta) = -F_{1,R}^{-1}\InParentheses{\frac{1}{1+\delta}} = \log\frac{1}{\delta}$.

    \paragraph{Squared Euclidean norm: $R(x) = \frac{1}{2}(x^2+(1-x)^2)$} In this case, we can verify that $F^{-1}_{1,R}(x)$ admits a closed-form for $x\in (0,1)$:
    \begin{align*}
        F_{1,R}(E) = \frac{1-E}{2},\forall E\in (-1,1) \Rightarrow F^{-1}_{1,R}(x) =1-2x, \forall x \in (0,1).
    \end{align*}
    Hence, $f_R(\delta) = -F_{1,R}^{-1}\InParentheses{\frac{1}{1+\delta}} = \frac{1-\delta}{1+\delta} \le 1$.

    \paragraph{Log barrier: $R(x) = -\log x-\log(1-x)$} In this case, we can verify that $F^{-1}_{1,R}(x)$ admits a closed-form for $x\in (0,1)$:
    \begin{align*}
         F^{-1}_{1,R}(x) = \frac{2x-1}{x^2-x}, \forall x \in (0,1).
    \end{align*}
    Hence, we have 
    \begin{align*}
        f_R(\delta) = -F_{1,R}^{-1}\InParentheses{\frac{1}{1+\delta}} = - \frac{2(1+\delta)-(1+\delta)^2}{1 - (1+\delta)} = \frac{1-\delta^2}{\delta} \le \frac{1}{\delta}.
    \end{align*}

    \paragraph{Negative Tsallis entropy: $R(x) = \frac{1}{1-\beta}(1-x^\beta+1-(1-x)^\beta)$ } In this case, we can verify that $F^{-1}_{1,R}(x)$ admits a closed-form for $x\in (0,1)$:
    \begin{align*}
         F^{-1}_{1,R}(x) = \frac{\beta}{1-\beta}\InParentheses{x^{\beta-1} - (1-x)^{\beta-1}}, \forall x \in (0,1).
    \end{align*}
    Hence, we have 
    \begin{align*}
         f_R(\delta) = -F_{1,R}^{-1}\InParentheses{\frac{1}{1+\delta}} = \frac{\beta}{1-\beta} \InParentheses{ \InParentheses{ \frac{\delta}{1+\delta}}^{\beta-1} - \InParentheses{ \frac{1}{1+\delta}}^{\beta-1}  } \le \frac{\beta}{1-\beta} \InParentheses{\frac{1+\delta}{\delta}}^{1-\beta} \le  \frac{2\beta}{1-\beta} \InParentheses{\frac{1}{\delta}}^{1-\beta}. 
    \end{align*}
    This completes the proof.
\end{proof}

\begin{theorem}\label{theorem:consecutive bad iterations}
    Assume the regularizer $R$ satisfies Assumption~\ref{assumption:standard} and Assumption~\ref{assumption:main}. For any $\delta \in (0, \hat{\delta})$ where $\hat{\delta}$ is a constant that depending only on the constants $c_1$ and $\delta'$ defined in \Cref{assumption:main}, the OFTRL dynamics on $A_\delta$ (defined in \eqref{eq:A delta}) with any step size $\eta \le \frac{1}{4L}$ satisfies the following: there is an iteration $T \le \frac{8+2L f_R(\delta)}{c_1c_3\eta L \delta}$ (with $f_R(\delta)$ defined in Lemma \ref{lemma: E delta}) such that for all $t \in \InBrackets{T, T+ \frac{c_1^2}{80\eta L \delta}}$, 
    \begin{align*}
        \DualGap(x^t, y^t) \ge c_2, 
    \end{align*}
    where $c_2 > 0$ is a constant defined in \Cref{assumption:main}.
\end{theorem}
\begin{proof}
    We will use notations and results presented in \Cref{theorem: main}. We restate some definitions and key facts:
    \begin{itemize}
        \item $T_1$ is the smallest iteration when $x^{T_1}[1] \ge \frac{1}{1+\delta}$; \\
        \item $T_2 > T_1$ is the smallest iteration when $y^{T_2}[1] \ge \frac{1}{2(1+\delta)}$;
        \item $T_h := \lfloor \frac{c_1}{2\eta L \delta} \rfloor \in [\frac{c_1}{2\eta L \delta}-1, \frac{c_1}{2\eta L \delta}]$;
        \item We have $\DualGap(x^t, y^t) \ge c_2$ for all $t \in [T_2 + \lceil \frac{c_1T_h}{20}\rceil, T_2 + \lfloor \frac{c_1T_h}{10} \rfloor-2]$. 
    \end{itemize}
    Notice that the interval $[T_2 + \lceil \frac{c_1T_h}{20}\rceil, T_2 + \lfloor \frac{c_1T_h}{10} \rfloor-2]$ has length at least
    \begin{align*}
        \left\lfloor \frac{c_1T_h}{10} \right\rfloor-2 - \left\lceil\frac{c_1T_h}{20}\right\rceil \ge \frac{c_1^2}{20\eta L \delta} -4 - \frac{c_1^2}{40\eta L\delta} = \frac{c_1^2}{40\eta L\delta}-4\ge \frac{c_1^2}{80\eta L \delta}.
    \end{align*}
    In the above inequality, we use $x -1 \le \lfloor x \rfloor\le  \lceil x\rceil \le x+1$ and $\delta$; and the fact that $\eta L \le \frac{1}{4}$ and $\delta \le \frac{c_1^2}{300}$. 
    
    Now we get $\Omega(\frac{1}{\delta})$ consecutive iterations with duality gap at least $c_2$. It remains to show that $T_2 = O(\frac{f_R(\delta)}{\delta})$. We proceed with two steps. We first upper bound $T_1$, then we use the obtained bound to further bound $T_2$. 
    \paragraph{Bounding $T_1$}
    By item 1 of \Cref{theorem: main}, we have $y^t[1] \le \frac{1}{2} -c_1$ for all $t \in [\frac{1}{2\eta L}, T_1-1]$. Recall that $e^t_y= 1-(1+\delta)x^t[1]$,  this implies $e^t_x \le -\frac{c_1}{2}$ for all $t \in [\frac{1}{2\eta L}, T_1-1]$. Since $T_1$ is the \emph{first} iteration that $x^{T_1}[1] \ge \frac{1}{1+\delta}$, it follows that
    \begin{align*}
       \frac{1}{1+\delta} \ge x^{T_1-1}[1] &= F_{\eta, R}\InBrackets{ E^{T_1-2}_x + e^{T_1-2}_x} \\
       &\ge F_{\eta,R}\InBrackets{\frac{1}{2\eta L} - \frac{c_1}{2} (T_1-2-\frac{1}{2\eta L}) + 1}\\
       &= F_{1,R}\InBrackets{\frac{1}{2 L} - \frac{\eta c_1}{2} (T_1-2-\frac{1}{2\eta L}) + \eta}
    \end{align*}
    By definition of $f_R(\delta)$ and monotonicity of $F_{1,R}$, we deduce
    \begin{align*}
        &\frac{1}{2L} - \frac{\eta c_1}{2} (T_1-2-\frac{1}{2\eta L}) + \eta \ge F^{-1}_{1,R}\InBrackets{\frac{1}{1+\delta}} \ge   - f_R(\delta) \\ 
        \Rightarrow &        T_1 \le \frac{2(\frac{1}{2 L}+\eta+f_R(\delta))}{\eta c_1} + 2 + \frac{1}{2\eta L}.
    \end{align*}
    Since $\eta L \le \frac{1}{4}$, the above simplifies to 
    \begin{align*}
        T_1 \le \frac{6}{c_1\eta L} + \frac{2f_R(\delta)}{\eta c_1}
    \end{align*}

    \paragraph{Bounding $T_2$}
    By item 2 of \Cref{theorem: main}, we have for all $t \in [T_1 +T_h, T_2]$
    \begin{align*}
        x^t[1] \ge \frac{1+c_3}{1+c_3+\delta}.  
    \end{align*}
    Note that $1+ c_3 + \delta \le 2$. This implies $e^t_y = 1 - (1+\delta)x^t[1] = -\frac{c_3\delta}{1+c_3 + \delta} \le -\frac{c_3\delta}{2}$ for all $T_1+T_h \le t \le T_2$. Moreover, for all $t \in [T_1+1, T_2]$,  we have $x^t[1] \ge \frac{1}{1+\delta}$ so $e^t_y \le 0$. Combining these gives
    \begin{align*}
        \frac{1}{2} \ge y^{T_2-1}[1] &= F_{\eta, R}\InBrackets{E^{T_2-1}_y +e^{T_2-1}_y} \\
        &= F_{\eta, R}\InBrackets{\sum_{t=1}^{T_1}e^t_x+\sum_{t=T_1+1}^{T_h-1} e^t_y + \sum_{t=T_h}^{T_2-1} e^t_y + e^{T_2-1}_y} \\
        &\ge F_{\eta, R}\InBrackets{T_1 -\frac{c_3\delta}{2}(T_2-T_h+1)}.
    \end{align*}
    By monotonicity of $F_{\eta,R}$ and $F_{\eta, R}[0] = \frac{1}{2}$ (\Cref{assumption:standard}), we have 
    \begin{align*}
        &T_1 -\frac{c_3\delta}{2}(T_2-T_h+1) \ge 0\\
        \Rightarrow & T_2 \le \frac{T_1}{c_3\delta} + T_h \\
        \Rightarrow & T_2 \le \frac{7+2L f_R(\delta)}{c_1c_3\eta L \delta}.
    \end{align*}

    Combining the facts that $\DualGap(x^t, y^t) \ge c_2$ for all $t \in [T_2 + \lceil \frac{c_1T_h}{20}\rceil, T_2 + \lfloor \frac{c_1T_h}{10} \rfloor-2]$, and the length of the interval is at least $\frac{c_1^2}{80\eta L \delta}$, we conclude that starting from iteration no more than $\frac{8+2L f_R(\delta)}{c_1c_3\eta L \delta}$, the following $\frac{c_1^2}{80\eta L \delta}$ iterations all have duality gap larger than the constant $c_2 > 0$.
\end{proof}

\subsubsection{Proof of \Cref{thm:random-lower-OMWU}}
By \Cref{theorem:consecutive bad iterations} and \Cref{lemma: E delta}, we know that we can find $T = O(\frac{\log(\frac{1}{\delta})}{\delta})$ such that for all $t \in [T, T_1]$ we have $\DualGap(x^t, y^t) \ge c_2$ where $T_1 = T + \Theta(\frac{1}{\delta}) = O(\frac{\log \frac{1}{\delta}}{\delta})$. Then we have 
\begin{align*}
    \sum_{t=1}^{T_1} \DualGap(x^t, y^t) \ge \sum_{t=T}^{T_1} \DualGap(x^t, y^t) = \Omega(\frac{1}{\delta}).
\end{align*}
Hence we get
\begin{align*}
    \frac{1}{T_1}  \sum_{t=1}^{T_1} \DualGap(x^t, y^t) = \Omega\InParentheses{\frac{1}{\log \frac{1}{\delta}}} = \Omega\InParentheses{\frac{1}{\log T_1}},
\end{align*}
where the last equality holds since $\log T_1 = \Theta(\log \frac{1}{\delta})$. We note that $T_1$ can be arbitrarily large as $\delta \rightarrow0$. Therefore, the uniform random-iterate convergence rate of OMWU is $\Omega(\frac{1}{\log T})$.

\subsubsection{Proof of \Cref{thm:random-lower-all-regularizers}}
\Cref{thm:random-lower-all-regularizers} follows in the same way as \Cref{thm:random-lower-OMWU} by combining \Cref{theorem:consecutive bad iterations} and \Cref{lemma: E delta} as we discussed in \Cref{sec:random proof overview}. 

\section{Missing Proofs in \Cref{sec:best-iterate}}\label{app:OMWU best-iterate}

\subsection{Global Phase: Missing Proofs in \Cref{sec:global phase}}\label{app:global phase}

\subsubsection{Proof of \Cref{prop:dynamic regret->best-iterate}}
\begin{proof}
    Recall that $\ell^t_x=Ay^t$ and $\ell^t_y=-A^\top x^t$. By definition, we have
    \begin{align*}
    &\sum_{t=1}^T\DualGap(x^t,y^t) = \sum_{t=1}^T\InParentheses{ \max_{y\in\Delta^{d_2}} (x^t)^\top Ay - \min_{x\in \Delta^{d_1}}x^\top Ay^t} \\
    & =  \sum_{t=1}^T \max_{z \in \+Z} \InAngles{F(z^t), z^t - z} \\
    &= \sum_{t=1}^T \InParentheses{\max_{y\in\Delta^{d_2}} (x^t)^\top Ay - (x^t)^\top Ay^t }  + \sum_{t=1}^T \InParentheses{(x^t)^\top Ay^t- \min_{x\in \Delta^{d_1}}x^\top Ay^t} \\
    &= \+R^T_y(\{y^t_\star\}) + \+R^T_x(\{x^t_\star\}).
\end{align*}
\end{proof}

\subsubsection{Proof of \Cref{lemma:OMWU interval}}
\begin{proof}
    By \Cref{lemma:OMWU rvu}, for any interval $\+I \in [s, e] \in [2, \infty]$\footnote{Here we ignore the first iteration which contributes at most $O(1)$ regret.}, we have
    \begin{align*}
        \+R^\+I_x+\+R^\+I_y = \max_{z} \sum_{t=s}^e \InAngles{F(z^t), z^t - z} \le \frac{1}{\eta}\sum_{t=s}^e \Theta^t-\Theta^{t+1} -\zeta^t \le \frac{1}{\eta}\Theta^s.
    \end{align*}
    For any $t \ge 2$, we can bound $\Theta^t$ by definition and the probability lower bound in \Cref{lemma:OMWU prob lower bound}.
    \begin{align*}
        \Theta^t &\le \KL(z^*, \hz^t) + \KL(\hz^t, z^{t-1}) \\
        &\le \sum_{i\in[d_1+d_2]}\log\frac{1}{\hz^t[i]} + \log\frac{1}{z^{t-1}[i]} \\
        &\le O\InParentheses{(d_1+d_2)\InParentheses{ \log(d_1d_2) + \frac{1}{\delta} }  }.
    \end{align*}
    Combining the above two inequalities completes the proof.
\end{proof}

\subsubsection{Proof of \Cref{lemma:OMWU variation}}
\begin{proof}
    Let $|\+I| = [s, e]$. Then by definition, we have
    \begin{align*}
        V^\+I &= \sum_{t=s+1}^e \max_{z} |\InAngles{F(z^t)-F(z^{t-1}),z}| \\
        &\le 2\sum_{t=s+1}^e \InNorms{F(z^t) - F(z^{t-1})}_\infty \tag{Cauchy-Shwarz and $\InNorms{z}_1\le 2$} \\
        &\le 2\sqrt{|\+I|\sum_{t=s+1}^e \InNorms{F(z^t) - F(z^{t-1})}_\infty^2} \tag{Cauchy-Shwarz}\\
        &\le  2\sqrt{|\+I|\sum_{t=s+1}^e \InNorms{z^t - z^{t-1}}_1^2} \tag{$F$ is $1$-Lipschitz since $A \in [0,1]^{d_1\times d_2}$}\\
        &\le O\InParentheses{\sqrt{|\+I|\log(d_1d_2)}}. \tag{\Cref{lemma:OMWU rvu} item 3}
    \end{align*}
\end{proof}

\subsubsection{Proof of \Cref{thm:OMWU dynamic}}
\begin{proof}
    Let $\+I_1 =[s_1, e_1], \ldots, \+I_M=[s_M, e_M]$ be any partition of the $T$ rounds. We let $z_m\in \argmin_{z} \sum_{t\in\+I_m} \InAngles{F(z^t), z}$.  Then, the social dynamic regret on $\+I_m$ is 
    \begin{align*}
        &\sum_{t\in \+I_m} \InAngles{F(z^t), z^t - z^t_\star} \\&= \sum_{t\in \+I_m} \InAngles{F(z^t), z^t -z^{s_m}_\star} + \sum_{t\in \+I_m} \InAngles{F(z^t),z^{s_m}_\star - z^t_\star} \\
        &\le \+R_z^{\+I_m} + 2|\+I_m|V^{\+I_m},
    \end{align*}
    where the last step is by the definition of interval regret and the fact that 
    \begin{align*}
       &\InAngles{F(z^t),z^{s_m}_\star - z^t_\star} \\
       &\le \InAngles{F(z^t) - F(z^{s_m}),z^{s_m}_\star} + \InAngles{F(z^{s_m}) - F(z^t), z^t_\star} \\
        &= \sum_{k=s_m+1}^t \InAngles{F(z^k) - F(z^{k-1}),z^{s_m}_\star}  \\
        &+ \sum_{k =s_m + 1}^t \InAngles{F(z^{k-1}) - F(z^k), z^t_\star}\\
        &\le 2 V^{\+I_m}.
    \end{align*}
    where the first inequality is by optimality of $z^{s_m}_\star$.
    Therefore, 
    \begin{align*}
        &\sum_{t=1}^T \InAngles{F(z^t), z^t - z^t_\star} \\
        &\le \sum_{m=1}^M (\+R^{I_m}_z + 2|\+I_m|V^{\+I_m}) \\
        &\le O\InParentheses{(d_1+d_2)\log(d_1d_2) \cdot \frac{M}{\eta\delta} + \max_{m\in[1,M]} |\+I_m| \cdot V^{[1,T]}} \tag{\Cref{lemma:OMWU interval}} \\
        &\le O\InParentheses{(d_1+d_2)\log(d_1d_2) \cdot \InParentheses{\frac{M}{\eta\delta} + \max_{m\in[1,M]} |\+I_m| \sqrt{T}}} \tag{\Cref{lemma:OMWU variation}}
    \end{align*}
    We choose $M = \max\{1, T^{\frac{3}{4}}\delta^{\frac{1}{2}}\}$ and make sure each interval has length $O(\frac{T}{M})$, then we have
    \begin{align*}
        \frac{M}{\delta} + \max_{m\in[1,M]} |\+I_m| \sqrt{T} = O\InParentheses{\frac{M}{\eta\delta} + \frac{T^{\frac{3}{2}}}{M}} = O\InParentheses{\frac{T^{\frac{3}{4}}\delta^{-\frac{1}{2}}}{\eta}}.
    \end{align*}
    We note that the above holds for any $T$. This completes the proof.
\end{proof}

\subsection{Initial Phase: Missing Proofs in \Cref{sec:initial phase}}\label{sec:OMWU best-initial}

\subsubsection{Proof of \Cref{prop:reduction}}
\begin{proof}
    Fix any $\eta \in (0, \eta']$. Since $A = b_1 \cdot \textbf{1} + b_2 A_{\delta_x, \delta_y}$, we know that the following two dynamics produce exactly the same trajectory $\{x^t, y^t\}$:
    \begin{itemize}
        \item[1.] OMWU with step size $\eta$ on $A$;
        \item[2.] OMWU with step size $b_2\eta$ on $A_{\delta_x, \delta_y}$.
    \end{itemize}
    The reason is (1) the update rule of OMWU concerns only the relative loss between actions so the $b_1 \textbf{1}$ component does not matter; (2) $\eta ( b_2 A_{\delta_x, \delta_y}) = (b_2\eta) A_{\delta_x, \delta_y}$. Therefore, if we assume that $\{x^t, y^t\}$ evaluated on $A_{\delta_x, \delta_y}$ has a rate $O(\frac{1}{b_2 \eta}f(T))$, them the same sequence $\{x^t, y^t\}$ evaluated on $A = b_1 \textbf{1} + b_2 A_{\delta_x, \delta_y}$ has a convergence rate of $O(b_2\cdot\frac{1}{b_2\eta}f(T)) = O(\frac{1}{\eta}f(T))$.
\end{proof}

\subsubsection{Proof of \Cref{thm:OMWU-initial}}

By \Cref{assumption:A2}, \Cref{lemma:structure of 2by2 game}, and \Cref{prop:reduction}, we only need to consider the following class of matrices without loss of generality:
\begin{align}\label{eq:A2}
    A_{\delta_x,\delta_y} = \begin{bmatrix}
        \frac{1-\delta_y}{1-\delta_x}  &\frac{1- \delta_x- \delta_y}{1-\delta_x}\\
        0 & 1
    \end{bmatrix}, \frac{1}{100} \le \delta_x \le \delta_y \le 1-\delta_x
\end{align}
For simplicity, in the proof, we omit the subscript and denote by $A$ the matrix $A_{\delta_x,\delta_y}$. We first summarize properties of $A$ that can be verified by simple algebra.
\begin{proposition}\label{prop:A2}
    Given \Cref{assumption:A2}, we have the following:
    \begin{itemize}
        \item The loss vectors of $A$ are
        \begin{align*}
            \ell_x = Ay= \InBrackets{ 1-\frac{\delta_y}{1-\delta_x} + \frac{\delta_x}{1-\delta_x}y[1], 1-y[1]}, 
            \ell_y = -A^\top x = - \InBrackets{\frac{1-\delta_y}{1-\delta_x}x[1], 1 - \frac{\delta_y}{1-\delta_x}x[1]}
        \end{align*}
        \item The loss vectors of $A$ satisfy: 
        \begin{align*}
            e_x :=\ell_x[1] - \ell_x[2] &= \frac{y[1] - \delta_y}{1-\delta_x} = \frac{1 - \delta_y - y[2]}{1-\delta_x},  e_y :=\ell_y[1] - \ell_y[2] = \frac{x[2] - \delta_x}{1-\delta_x} = \frac{ 1- \delta_x- x[1]}{1-\delta_x}.
        \end{align*} 
        Moreover, 
        \begin{align*}
            e_x \in \InBrackets{-\frac{\delta_y}{1-\delta_x}, 1} \subseteq [-1,1],  e_y \in \InBrackets{-\frac{\delta_x}{1-\delta_x}, 1} \subseteq [-2\delta_x, 1].
        \end{align*}
    \end{itemize}
\end{proposition}

We consider two cases: (1) $\delta_y \le \frac{1}{2}$; (2) $\delta_y > \frac{1}{2}$. We prove \Cref{thm:OMWU-initial} for each case in the following.

\paragraph{Case 1: $\delta_y \ge \frac{1}{2}$.} 

\begin{lemma}\label{lemma:A2 case1}
    Consider matrix $A$ defined in \Cref{eq:A2} that satisfies $\delta_y \in [\frac{1}{2}, 1)$. Let $\{x^t,y^t\}_{t\ge 1}$ be the iterates produced by OMWU dynamics with uniform initialization and step size $\eta \le 1$. Let $T_1$ be the first iteration such that $x^t[1] \ge 1-\delta_x$. Then 
    \begin{itemize}
        \item[1.]  $\DualGap(x^t,y^t) \le \frac{x^t[2]}{x^t[1]} \le 9 \exp(-\frac{\eta t}{42})$, for all $t \in [1, T_1-1]$;
        \item[2.] $\DualGap(x^{T_1},y^{T_1})\le 2\delta_x$.
        \item[3.] For all $t \in [1, T_1]$, there exists a universal constant $C > 0$ such that 
        \[\min_{k\in[1,t]} \DualGap(x^k,y^k) \le \frac{C}{\eta} \frac{1}{t}.\] 
    \end{itemize}
\end{lemma}
\begin{proof}
The OMWU dynamics is initialized with uniform distributions $(x^1, y^1)$ and step size $\eta \le 1$. Denote by $T_0 = \lceil \frac{1}{\eta} \rceil + 1 \le \frac{2}{\eta}$. 
    Using the update rule, we have for all $t \in [1, T_0]$,
    \begin{align*}
        \frac{x^{t}[1]}{x^{t}[2]} & = \frac{x^1[1]}{x^1[2]} \cdot \exp\InParentheses{-\eta E^{t-1}_x - \eta e^{t-1}_x} \le e^{\eta t} \le e^2,
    \end{align*}
    where we use $x^1[1] = x^1[2] = \frac{1}{2}$ and $-e_x \le 1$ by \Cref{prop:A2}. This implies $x^t[1]\le \frac{e^2}{e^2+1} < \frac{8}{9}$ for all $t \in [1,T_0]$. We define $T_1 > T_0$ the first iteration where $x^t[1] \ge 1-\delta_x$. 

    We have the following two inequalities on $e^t_y = \ell^t_y[1] - \ell^t_y[2]$ by \Cref{prop:A2}:
    \begin{align}
        e^t_y &= \frac{1-\delta_x-x^t[1]}{1-\delta_x} \ge 0, \forall t \in [1, T_1-1], \label{eq:ety1} \\
        e^t_y &= \frac{1-\delta_x-x^t[1]}{1-\delta_x} \ge  \frac{1}{10}, \forall t \in [1, T_0].\label{eq:ety2}
    \end{align}
    where in \eqref{eq:ety1} we use $x^t[1] < 1- \delta_x$ for all $t \in [1, T_1-1]$; in \eqref{eq:ety2} we use $\delta_x \le \frac{1}{100}$ and $x^t[1] \ge \frac{8}{9}$ for all $t \in [1, T_0]$.
    For any $t \in [T_0, T_1]$, we have $y^t$ satisfies
    \begin{align*}
        \frac{y^t[1]}{y^t[2]} & =  \frac{y^1[1]}{y^1[2]} \cdot \exp\InParentheses{-\eta E^{T_0-1}_y - \sum_{k=T_0}^{t-1} e^t_y - e^{t-1}_y}                                               \\
        & \le \exp\InParentheses{-\eta E^{T_0-1}_y} \tag{by \eqref{eq:ety1}}                                                                              \\
        & \le \exp\InParentheses{-\frac{\eta (T_0 - 1)}{10}} \tag{by \eqref{eq:ety2}} \\
        & \le \exp\InParentheses{-\frac{1}{10}} < \frac{10}{11}. \tag{$T_0-1 \ge \frac{1}{\eta}$}
    \end{align*}
    This implies $y^{t}[1] < \frac{10}{21}$ for all $t \in [T_0, T_1]$. Moreover, for all $t \in [T_0, T_1]$, we have $e^t_x$ satisfies
    \begin{align*}
        e^t_x = \frac{y^t[1] - \delta_y}{1-\delta_x}         \le \frac{\frac{10}{21} - \frac{1}{2}}{1-\delta_x} \le -\frac{1}{42},
    \end{align*}
    where we use $y^t[1] \le \frac{10}{21}$ and $\delta_y \ge \frac{1}{2}$.

    For all $t \in [T_0, T_1-1]$, using the fact that $e^t_x \in [-1,-\frac{1}{42}]$, we have $x^t$ satisfies
    \begin{align*}
        \frac{x^t[1]}{x^t[2]} & = \frac{x^1[1]}{x^1[2]}\cdot \exp\InParentheses{- 2\eta e^{t-1}_x -  \sum_{k=1}^{T_0-1} \eta e^k_x - \sum_{k=T_0}^{t-2} \eta e^k_x }                                      \\
        & \ge \exp\InParentheses{ \frac{\eta(t- T_0)}{42} - \eta T_0} \tag{$e^{t}_x \le -\frac{1}{42}$ for $t \in [T_0, T_1]$ and $e^t_x \le -1$ for all $t$}  \\
        & \ge  \exp\InParentheses{ \frac{\eta t}{42}-\frac{1}{21} - 2}. \tag{$T_0 \le \frac{2}{\eta}$}                                                   \\
        & \ge \frac{1}{9} \cdot \exp\InParentheses{ \frac{\eta t}{42}} \tag{$e^{--\frac{1}{21}-2} \ge \frac{1}{9}$}.
    \end{align*}

    Now we track the duality gap. Note that for $t \in [T_0, T_1-1]$, we have $x^t[1] \le 1-\delta_x$ and $y^1[1] \le \frac{10}{21} \le \delta_y$. Therefore,
    \begin{align*}
        \DualGap(x^t, y^t) & = \max_{i\in \{1,2\}} (A^\top x^t)[i] - \min_{i\in \{1,2\}} (A y^t)[i] \\
        & = \frac{\delta_y}{1-\delta_x}(1-x^t[1]) - \frac{\delta_x}{1-\delta_x}y^t[1]           \\
        & \le \frac{\delta_y}{1-\delta_x}(1-x^t[1])                                                   \\
        & \le x^t[2]  \tag{$0 < \delta_y \le 1-\delta_x$}\\
        &\le  \frac{x^t[2]}{x^t[1]}.
    \end{align*}
    Then we get for all $t \in [T_0, T_1-1]$,
    \begin{align*}
        \DualGap(x^t, y^t) \le \frac{x^t[2]}{x^t[1]} \le 9 \exp\InParentheses{ -\frac{\eta t}{42}}.
    \end{align*}
    Since $T_0 < \frac{2}{\eta}$, the above bounds on duality gap also holds for all $t \in [1, T_0]$. Thus we conclude that for all $t \in [1, T_1-1]$, $\DualGap(x^t, y^t)\le 9 \cdot \exp\InParentheses{-\frac{\eta t}{42}}$. 
    
    Moreover, since $x^{T_1}[1] \ge 1-\delta_x$ and $y^{T_1}[1] < \frac{10}{21} \le \delta_x$, we have
    \begin{align*}
        \DualGap(x^{T_1}, y^{T_1}) = \max_{i\in \{1,2\}} (A^\top x^t)[i] - \min_{i\in \{1,2\}} (A y^t)[i] 
        = \frac{1-\delta_y}{1-\delta_x}x^{T_1}[1] -1 + \frac{\delta_y}{1-\delta_x}  \le \frac{1}{1-\delta_x} - 1 \le 2\delta_x.
    \end{align*}
    To summarize, we have shown the following:
    \begin{itemize}
        \item  For $t \in [1, T_1-1]$, we have  linear convergence rate: $\DualGap(x^t, y^t) \le 9\exp(-\frac{\eta t}{42})$;
        \item At $T_1$, we have $\DualGap(x^{T_1}, y^{T_1}) \le 2\delta_x$.
    \end{itemize}
    For all $t \in [1, T_1-1]$, we know there exists a universal constant $C'$ such that $9\exp(-\frac{\eta t}{42}) \le \frac{C'}{\eta t}$. Thus we have $\min_{k\in[1,t]} \DualGap(x^k,y^k) \le \frac{C}{\eta} \frac{1}{t}$ for all $t \in [1, T_1 -1]$. For $t = T_1$, we have 
    \begin{align*}
        \min_{k\in[1,T_1]} \DualGap(x^k,y^k) \le \min_{k\in[1,T_1-1]} \DualGap(x^k,y^k) \le \frac{C'}{\eta (T_1 - 1)} \le \frac{2C'}{\eta T_1}.
    \end{align*}
    Let $C = 2C'$, we conclude that for all $t \in [1, T_1]$, it holds that \[\min_{k\in[1,t]} \DualGap(x^k,y^k) \le \frac{C}{\eta} \frac{1}{t}.\] 
    This completes the proof.
\end{proof}

\paragraph{Case 2: $\delta_y < \frac{1}{2}$.} 

\begin{lemma}\label{lemma:A2 case2}
    Consider matrix $A$ defined in \eqref{eq:A2} that satisfies \Cref{assumption:A2} and $\delta_y \in (0, \frac{1}{2})$. Let $\{x^t,y^t\}_{t\ge 1}$ be the iterates produced by OMWU dynamics with uniform initialization and step size $\eta \le \frac{1}{10}$. Denote $T_1$ the first iteration that $x^t[1] \ge 1-\delta_x$. Then 
    \begin{itemize}
        \item[1.] $\DualGap(x^{T_1},y^{T_1})\le 2\delta_x$.
        \item[2.] For all $t \in [1, T_1]$, there exists a universal constant $C > 0$ such that 
        \[\min_{k\in[1,t]} \DualGap(x^k,y^k) \le \frac{C}{\eta} \frac{\log^2t}{t}.\] 
    \end{itemize}
\end{lemma}
\begin{proof}
    The proof of \Cref{lemma:A2 case2} is more involved than that of \Cref{lemma:A2 case1}. We track the trajectory of OMWU dynamics in two phases:
    \begin{itemize}
        \item \textbf{Phase I:} $x^t[1]$ decreases and $y^t[1]$ decreases. At the end of phase I, $y^t[1] \le \delta_y$ and the duality gap decreases to $O(\delta_y)$.
        \item \textbf{Phase II:} $x^t[1]$ increases and $y^t[1]$ continues to decrease. At the end of phase II, $x^t[1] \ge 1-\delta_x$ and the duality gap further decreases to $O(\delta_x)$.
    \end{itemize}
    \paragraph{Phase I: $x^t[1]$ and $y^t[1]$ both decreases.} Recall that initially, $x^t[1]=y^t[1] =\frac{1}{2}$. We note that $e^t_x = \frac{y^t[1]-\delta_y}{1-\delta_x}$, thus as long as $y^t[1] \ge \delta_y$, we have $e^t_x \ge 0$. This implies $x^t[1] \le x^1[1]=\frac{1}{2}$ until the first iteration when $y^t[1]<\delta_y$, which we denote as $T_y$. Now we know for all iterations $t \in [1, T_y]$, $x^t[1]\le \frac{1}{2}$, this further implies $e^t_y = 1-\frac{x^t[1]}{1-\delta_x} \ge \frac{1}{3}$ since $0<\delta_x\le \frac{1}{100}$. As a result, we have 
    \begin{align*}
        \frac{y^t[1]}{y^t[2]} = \exp\InParentheses{ -\eta E^{t-1}_y-\eta e^{t-1}_y} \le \exp\InParentheses{-\frac{\eta t}{3}}, \forall t \in [1, T_y].
    \end{align*}
    Moreover, we know at time $T_y$, $y^{T_y}[1] \le \delta_y$ for the first time. Consider the duality gap of $\{(x^t, y^t)\}_{t\in[1,T_y-1]}$: since $x^{t}[1]\le 1-\delta_x$ and $y^{t}[1] \ge \delta_y$, we have $\forall t \in [1, T_y-1]$,
    \begin{align*}
        \DualGap(x^{t}[1], y^t[1])=\max_{i\in \{1,2\}} (A^\top x^t)[i] - \min_{i\in \{1,2\}} (A y^t)[i] = y^t[1] - \frac{\delta_y}{1-\delta_x}x[1] \le y^t[1] \le \frac{y^t[1]}{y^t[2]}\le  \exp\InParentheses{-\frac{\eta t}{3}},
    \end{align*}
    Thus we have a linear convergence for $t\in[1,T_y-1]$. For $t= T_y$, since $x^{t}[1]\le 1-\delta_x$ and $y^{t}[1] \le \delta_y$, the duality gap at $(x^{T_y},y^{T_y})$ is 
    \begin{align*}
        \DualGap(x^{T_y}[1], y^{T_y}[1])& =\max_{i\in \{1,2\}} (A^\top x^t)[i] - \min_{i\in \{1,2\}} (A y^t)[i] \\
        & =\frac{\delta_y}{1-\delta_x}(1-x^{T_y}[1]) - \frac{\delta_x}{1-\delta_x}y^{T_y}[1] \\
        & \le \frac{\delta_y}{1-\delta_x} \le 2\delta_y.
    \end{align*}
    This completes the analysis of Phase I.

    \paragraph{Phase II: $y^t[1]$ continues to decrease but $x^t[1]$ increases.} Now that at time $T_y$, $y^{T_y}\le \delta_y$, $e^t_x$ could be negative since then and $x^t[1]$ might increase. However, we know that $e^t_y=1-\frac{x^t[1]}{1-\delta_x} \ge 0$ holds as long as $x^t[1] \le 1-\delta_x$. Thus $y^t[1]$ would continue to decrease until the first iteration $x^t[1] > 1-\delta_x$, which we denote as $T_x$. Now recall that $x^{T_y}\le \frac{1}{2}$, we argue $x^{t}\le \frac{3}{4}$ for all $t \le T_m:= T_y+\lceil \frac{2}{\eta} \rceil$:
    \begin{align*}
        \frac{x^{t}[1]}{x^{t}[2]}=\exp\InParentheses{-\eta E^{T_y}_x-\eta \sum_{j=T_y+1}^{t-1}e^j_x - \eta e^{t-1}_x} \le \exp\InParentheses{\eta (t-T_y)} \le  e^2, \forall t\in[T_y+1, T_m]
    \end{align*}
    where we use $e^t_x \ge 0$ for all $t \in [1,T_y]$ and $e^t_x \ge -1$ for all $t$ as well as $t - T_y \le \frac{2}{\eta}$. This implies $x^t[1]\le \frac{e^2}{e^2+1}\le\frac{8}{9}$ for all $t \in [1, T_m]$. It further implies $e^t_y=1-\frac{x^t[1]}{1-\delta_x} \ge \frac{1}{10}$ since $\delta_x\le \frac{1}{100}$. 

    Using $T_m- T_y\ge \lceil \frac{2}{\eta} \rceil$ and bounds on $e^t_y$, we have 
    \begin{align*}
       \forall t \in [T_m, T_x-1], \quad \frac{y^{t}[1]}{y^{t}[2]} &= \exp\InParentheses{\eta E^{T_y-1}_y -\eta \sum_{k=T_y}^{t-1} e^k_y - \eta e^{t-1}_y } \\
        &= \exp\InParentheses{-\eta E^{T_y-1}_y -\eta e^{T_y-1}_y -\eta \sum_{k=T_y}^{t-1} e^k_y - \eta e^{t-1}_y + \eta e^{T_y-1}_y } \\
        &\le  \frac{\delta_y}{1-\delta_y} \cdot \exp\InParentheses{-\eta \sum_{k=T_y}^{t-1} e^k_y - \eta e^{t-1}_y + \eta e^{T_y-1}_y } \\
        &\le \frac{\delta_y}{1-\delta_y} \cdot \exp\InParentheses{-\eta \sum_{k=T_y}^{T_m-1} e^k_y + \eta e^{T_y-1}_y } \\
        &\le \frac{\delta_y}{1-\delta_y} \cdot \exp\InParentheses{-\frac{\eta(T_m-T_y)}{10} + \eta } \\
        &\le \frac{\delta_y}{1-\delta_y} \cdot \exp\InParentheses{-\frac{1}{10}} \le \frac{10}{11} \cdot \frac{\delta_y}{1-\delta_y}.
    \end{align*}
    where: in the first inequality, we use $\exp\InParentheses{-\eta E^{T_y-1}_y -\eta e^{T_y-1}_y}=\frac{y^{T_y}[1]}{y^{T_y}[2]}\le \frac{\delta_y}{1-\delta_y}$; in the second inequality  we use $e^t_y \ge0$ for all $t < T_x$; in the third inequality, we use $e^t_y \ge \frac{1}{10}$ for all $t \in [1, T_m]$. In the second last inequality, we use $T_m-T_y\ge \frac{2}{\eta}$ and $\eta \le \frac{1}{10}$. This implies \begin{align*}
        y^t[1]=\frac{1}{1+y^t[2]/y^t[1]} \le \frac{1}{1+\frac{11(1-\delta_y)}{10\delta_y}} = \frac{10\delta_y}{11-\delta_y}\le \frac{20}{21}\delta_y, \forall t \in [T_m,T_x-1]
    \end{align*}
    where we use $\delta_y < \frac{1}{2}$ in the last inequality. Then $e^t_x = \frac{y^t[1]-\delta_y}{1-\delta_x} \le -\frac{\delta_y}{21}$ for all $t \in [T_m, T_x-1]$.

    Given that $e^t_x \le -\frac{\delta_y}{21}$ for all $t \in [T_m, T_x-1]$, we know $x^t[1]$ increases during these iterations. In order to quantify the speed, we first provide some bounds on $T_m = T_y + \lceil\frac{2}{\eta}\rceil$. In the analysis in Phase I, we know (1) $\frac{y^t[1]}{y^t[2]} \le \exp\InParentheses{-\frac{\eta t}{3}}, \forall t \in [1, T_y]$ and (2) $y^t[1]\ge \delta_y$ for all $t \in [1, T_y-1]$. Combining the two inequalities gives
    \begin{align*}
        \frac{\delta_y}{1-\delta_y} \le \exp\InParentheses{-\frac{\eta (T_y-1)}{3}} \Rightarrow T_y \le \frac{3}{\eta}\log \frac{1}{\delta_y} +1.
    \end{align*}
    This implies $T_m -1 = T_y -1 + \lceil\frac{2}{\eta}\rceil \le \frac{4}{\eta} \log \frac{1}{\delta_y}$
    
    Now we analysis $x^t$ for $t \in [T_m, T_x]$:
    \begin{align*}
        \frac{x^t[1]}{x^t[2]} &= \exp\InParentheses{-\eta E^{t-1}_x - \eta e^{t-1}_x} \\
        &= \exp\InParentheses{-\eta E^{T_m-1}_x -\eta \sum_{k=T_m}^{t-1} \eta e^{k}_x - \eta e^{t-1}_x} \\
        & \ge \exp\InParentheses{-\eta (T_m-1) + \frac{\eta \delta_y(t-T_m + 1)}{21}} \\
        & \ge \exp\InParentheses{ -2\eta (T_m-1) + \frac{\eta \delta_y t}{21}} \\
        &\ge \delta_y^8\cdot\exp\InParentheses{\frac{\eta \delta_y t}{21}}.
    \end{align*}
    where: in the first inequality we use $e^t_x \le 1$ for all $t$ and $e^t_x\le - \frac{\delta_x}{21}$ for $t \in [T_m, T_x-1]$; in the second inequality we use $\delta_y \le 1$; in the last inequality, we use $T_m-1\le \frac{4}{\eta}\log\frac{1}{\delta_y}$.

    Now we analyze the duality gap. Since for all $t \in [T_m,T_x-1]$, we have $x^t[1] \le 1-\delta_x$ and $y^t[1]\le \delta_y$, we have 
    \begin{align*}
        \DualGap(x^t, y^t) & = \max_{i\in \{1,2\}} (A^\top x^t)[i] - \min_{i\in \{1,2\}} (A y^t)[i]\\
        &=\frac{\delta_y}{1-\delta_x}(1-x^t[1]) - \frac{\delta_x}{1-\delta_x}y^t[1]           \\
        & \le \frac{\delta_y}{1-\delta_x}(1-x^t[1])                                                   \\
        & \le x^t[2]  \tag{$0 < \delta_y \le 1-\delta_x$}\\
        &\le  \frac{x^t[2]}{x^t[1]} \\
        &\le \frac{1}{\delta_y^8}\cdot \exp\InParentheses{-\frac{\eta\delta_y t}{21}}.
    \end{align*}
    For $t= T_x$, since $x^t[1] \ge 1-\delta_x$ and $y^t[1]\le \delta_y$, we have 
    \begin{align*}
        \DualGap(x^{T_x}, y^{T_x}) = \max_{i\in \{1,2\}} (A^\top x^t)[i] - \min_{i\in \{1,2\}} (A y^t)[i]
        = \frac{1-\delta_y}{1-\delta_x}x^{T_x}[1] -1 + \frac{\delta_y}{1-\delta_x}  \le \frac{1}{1-\delta_x} - 1 \le 2\delta_x.
    \end{align*}

    \paragraph{Combining bounds.} We combine the analysis in Phase I and II here. We claim the iterates $\{x^t,y^t\}_{t\in[1,T_x]}$ has the following convergence rates.
    \begin{claim}\label{claim:A2 case2}
        There exists a universal constant $C$ such that for any $t \in [1,T_x]$, the following best-iterate convergence holds.
        \begin{align*}
            \min_{k\in[1,t]}\DualGap(x^k,y^k) = \frac{C}{\eta} \cdot \frac{\log^2t}{t}.
        \end{align*}
    \end{claim}
    \begin{proof}
        We prove the claim by considering the following cases.
        \paragraph{1. $t\in[1,T_y]$} In this case, by analysis in Phase I, we have $\DualGap(x^t,y^t) \le \exp\InParentheses{-\frac{\eta t}{3}}$. Then there exists a universal constant $C_1 > 0$ such that
            \begin{align*}
                \DualGap(x^t,y^t) \le \exp\InParentheses{-\frac{\eta t}{3}} \le \frac{C_1}{\eta t}.
            \end{align*}
        \paragraph{2. $t \in [T_y, T_m]$} We note that by analysis in Phase I, $\DualGap(x^{T_y}, y^{T_y}) \le 2\delta_y$; by analysis in Phase II, we have $T_m \le \frac{4}{\eta} \log\frac{1}{\delta_y}$. Thus for $t \in [T_y, T_m]$, we have
        \begin{align*}
            \min_{k\in [1,t]} \DualGap(x^k,y^k) &\le \DualGap(x^{T_y}, y^{T_y}) \\
            &\le 2\delta_y \\
            & = \frac{8}{\eta} \frac{1}{ \frac{4}{\eta} \frac{1}{\delta_y} } \\
            &\le \frac{8}{\eta} \cdot\frac{1}{\frac{4}{\eta} \log\frac{1}{\delta_y}}  \tag{$a\ge \log a, \forall a>1 $}\\
            &\le \frac{8}{\eta} \frac{1}{T_m}\tag{$T_m\le \frac{4}{\eta}\log\frac{1}{\delta_y}$}\\
            &\le \frac{8}{\eta}\cdot \frac{1}{t} \tag{$t\le T_m$}.
        \end{align*}

        \paragraph{3. $t \in [T_m, T_x-1]$} By analysis in Phase II and duality bound of $(x^{T_y},y^{T_y})$, we have for any $t\in [T_m, T_x-1]$, there exists a universal constant $C_2 = 8\times21$ such that
        \begin{align*}
            \min_{k\in[1,t]} \DualGap(x^k,y^k) \le \min\left\{2\delta_y, \frac{1}{\delta_y^8}\exp\InParentheses{-\frac{\eta\delta_y t}{21}}\right\} 
            &\le \frac{C_2}{\eta} \cdot \frac{\log^2t}{t}.
        \end{align*}
        The last inequality holds since (1) if $\delta_y \le \frac{C_2 \log^2 t}{\eta t}$, then the inequality holds; (2) if $\delta_y \ge \frac{C_2}{\eta} \frac{\log^2 t}{t}$, then we have 
        \begin{align*}
            &\delta_y \ge \frac{C_2}{\eta} \frac{\log^2 t}{t} \\
            \Rightarrow & \frac{C_2}{\eta} \frac{\log \frac{1}{\delta_y}}{\delta_y} \le \frac{t}{\log^2 t} \log \frac{t}{\frac{C_2}{\eta}\log t} \le \frac{t}{\log t} \tag{$\frac{C_2}{\eta} \ge 1$}\\
            \Leftrightarrow & 8 \log t + 8 \log \frac{1}{\delta_y}  \le \frac{\eta \delta_y t}{21}  \tag{$C_2 = 8 \times 21$} \\
            \Rightarrow & \log t \le 8\log \delta_y +\frac{\eta \delta_y t}{21} \\
             \Leftrightarrow & t \le \delta_y^8 \exp\InParentheses{\frac{\eta\delta_y t}{21}}\\
             \Leftrightarrow & \frac{1}{\delta_y^8} \exp\InParentheses{-\frac{\eta\delta_y t}{21}} \le \frac{1}{t}.
         \end{align*}

         \paragraph{4. $t=T_x$} We know $\DualGap(x^{T_x},y^{T_x}) \le 2\delta_x$. Let $C' = \max\{C_1, C_2\}$. We can use the above bounds and conclude
         \begin{align*}
            \min_{k\in[1,T_x]} \DualGap(x^k,y^k) &\le \DualGap(x^{T_x-1}, y^{T_x-1}) \\
            &\le \frac{C'}{\eta} \frac{\log^2(T_x-1)}{T_x-1}\\
            &\le  \frac{2C'}{\eta} \frac{\log^2(T_x)}{T_x}\
         \end{align*}
         Thus the claim holds with $C = 2C' = 2\max\{C_1, C_2\}$.
    \end{proof}
    Let $T_1 = T_x$, \Cref{lemma:A2 case2} then follows from \Cref{claim:A2 case2} and the fact that $\DualGap(x^{T_x}, y^{T_x}) \le 2\delta_x$.
\end{proof}
Combining \Cref{lemma:A2 case1} and \Cref{lemma:A2 case2} gives \Cref{thm:OMWU-initial}.

\subsection{Combining Two-Phase Analysis: Proof of \Cref{thm:best iterate OMWU}}
Consider any matrix game $A \in [0,1]^{2 \times 2}$ that has a fully-mixed Nash equilibrium with minimum probability $\delta > 0$. First, by \Cref{thm:OMWU-initial}, we know there exists an initial phase $[1, T_1]$ such that $O(\frac{\log^2T}{T})$ best-iterate convergence rate holds and $\DualGap(x^{T_1}, y^{T_1}) \le 2\delta$. 

Thus we only need to consider iterations $t \ge T_1+1$. By \Cref{thm:OMWU dynamic}, we know for all iteration $t$, the best-iterate convergence rate of $O(T^{-\frac{1}{4}}\delta^{-\frac{1}{2}})$ holds (note that $d_1=d_2=2$ are constants). Then we have for all $T \ge T_1+1$
\begin{align*}
    \min_{t \in [1,T]} \DualGap(x^t, y^t) \le \min\{2\delta, O(T^{-\frac{1}{4}}\delta^{-\frac{1}{2}})\} \le O(T^{-\frac{1}{6}}).
\end{align*}
This completes the proof.

\end{document}